\documentclass[11pt]{article}
\usepackage{enumerate}
\usepackage[OT1]{fontenc}
\usepackage{amsmath,amssymb}
\usepackage{mathtools}
\usepackage[numbers]{natbib}
\usepackage[usenames]{color}
\usepackage{pgfplots}
\usepackage{smile}
\usepackage{multirow}
\usepackage{rotating}
\usepackage{enumerate}
\usepackage{esvect}
\usepackage{tikz}
\usetikzlibrary{patterns}
\usetikzlibrary{arrows}

\usepackage{times}
\usepackage{setspace}
\usepackage{amsfonts}       %
\usepackage{nicefrac}       %
\usepackage{microtype}      %

\usepackage[colorlinks,
            linkcolor=red,
            anchorcolor=blue,
            citecolor=blue
            ]{hyperref}
\usepackage{algorithm}
\usepackage{algorithmic}

\usepackage{mathrsfs}
\usepackage{fullpage}
\renewcommand{\baselinestretch}{1.25}

\newcommand{\fullmodel}{Optimism-Pessimism Principle for Persuasion Process}
\newcommand{\ORAI}{\texttt{OP4}}
\newcommand{\pstate}{outcome}
\newcommand{\mstate}{state}
\newcommand{\cstate}{context}%

\def \Ps {\operatorname{Pers}}
\def \OPT {\operatorname{OPT}}
\def \Reg {\operatorname{Reg}}
\def \Gap {\operatorname{Gap}}

\def \poly {\operatorname{poly}}

\theoremstyle{mytheoremstyle}

\newcommand{\compilehidecomments}{false}%
\ifthenelse{ \equal{\compilehidecomments}{true} }{%
	\newcommand{\hf}[1]{}
    \newcommand{\jw}[1]{}
    \newcommand{\zhuoran}[1]{}
    \newcommand{\zf}[1]{}
    \newcommand{\hfr}[1]{}
    \newcommand{\zx}[1]{}
    \newcommand{\todo}[1]{}
}{
\usepackage{xcolor}
\definecolor{darkspringgreen}{rgb}{0.09, 0.45, 0.27}
\definecolor{arsenic}{rgb}{0.23, 0.27, 0.29}
\newcommand{\hf}[1]{{\color{blue}  [\text{Haifeng:} #1]}}
\newcommand{\hfr}[1]{{\color{blue}    #1}}
\newcommand{\jw}[1]{{\color{darkspringgreen}  [\text{Jibang:} #1]}}
\newcommand{\zhuoran}[1]{{\color{arsenic}  [\text{zhuoran:} #1]}}
\newcommand{\zf}[1]{{\color{cyan}  [\text{ZF:} #1]}}
\newcommand{\zx}[1]{{\color{purple}  [\text{Zixuan:} #1]}}
\newcommand\todo{\textcolor{red}}
}

\title{Sequential Information Design: Markov Persuasion Process and Its Efficient Reinforcement Learning}

\date{}

\author{Jibang Wu\thanks{University of Virginia. Email: \texttt{jw7jb@virginia.edu}.}\qquad
Zixuan Zhang\thanks{University of Science and Technology of China. Email: \texttt{zhangzixuan@mail.ustc.edu.cn}.}\qquad
Zhe Feng\thanks{Google. Email: \texttt{zhef@google.com}.}\qquad
Zhaoran Wang\thanks{Northwestern University. Email: \texttt{zhaoranwang@gmail.com}.} 
\\ 
Zhuoran Yang\thanks{Yale University. Email: \texttt{zhuoran.yang@yale.edu}.}\qquad
Michael I. Jordan\thanks{UC Berkeley. Email: \texttt{jordan@cs.berkeley.edu}.}\qquad
Haifeng Xu\thanks{University of Virginia. Email: \texttt{hx4ad@virginia.edu}.}}

\begin{document}

\maketitle

\begin{abstract}
In today's economy, it becomes important for Internet platforms to consider the sequential information design problem to align its long term interest with incentives of the gig service providers (e.g., drivers, hosts). This paper proposes a novel model of sequential information design, namely the Markov persuasion processes (MPPs).
Specifically, in an MPP,  a sender, with informational advantage, seeks to persuade a stream of myopic receivers to take actions that maximizes the sender's cumulative utilities in a finite horizon Markovian environment with varying prior and utility functions. Planning in MPPs thus faces the unique challenge in finding a signaling policy that is simultaneously persuasive to the myopic receivers and inducing the optimal  long-term cumulative utilities of  the sender. Nevertheless, in the population level where the model is known, it turns out that we can efficiently determine the optimal (resp. $\epsilon$-optimal) policy with finite (resp. infinite) states and outcomes, through a modified formulation of the Bellman equation that additionally takes persuasiveness into consideration.

Our main technical contribution is to study the MPP under the online reinforcement learning (RL) setting, where the goal is to learn the optimal signaling policy by interacting with with the underlying MPP, without the knowledge of the sender's utility functions, prior distributions, and the Markov transition kernels.  
For such a problem, we design a provably efficient no-regret learning algorithm, the \fullmodel{} (\ORAI), which features a novel combination of both optimism and pessimism principles. In particular, we  obtain optimistic  estimates of the value functions to  encourage  exploration under the unknown environment. Meanwhile, we additionally robustify the  signaling policy  with respect to the uncertainty of   prior estimation to prevent receiver's detrimental equilibrium behavior. Our algorithm enjoys  sample efficiency by achieving a sublinear $\sqrt{T}$-regret upper bound. Furthermore, both our algorithm and theory can be applied to MPPs with large space of outcomes and states via function approximation, and we showcase such a success under the linear setting. 

\end{abstract}

\section{Introduction}
\label{sec:motivating-exp}
Most sequential decision models assume that there is a sole agent who possesses and processes all relevant (online or offline) information and takes an action accordingly.
However, the economic literature on information design \citep{kamenica2011bayesian, bergemann2019information} highlights the importance of considering information asymmetry   in decision making, where the decision maker and information possessor may be two parties having different interests and goals. 
For example, a ride-sharing platform holds historical and real-time data on active riders and driver types in different locations, based on which they have developed centralized combinatorial optimization algorithms and reinforcement learning algorithms for vehicle repositioning, routing and order matching to optimize their operational efficiency and profit~\cite{li2019efficient,qin2020ride, liang2021integrated, qin2021reinforcement}. But the de facto decision makers are the drivers. Moreover, as increasingly many drivers are freelancers instead of employees, the platform cannot expect to give mandatory orders to them. On the other hand, if the platform shares no information on rider demand, most drivers will not be able to efficiently find profitable trips. Therefore, it is not only realistic but also necessary to consider an information design problem that aligns the interests of the two parties in sequential decision making processes of this kind.

Given the large data sets being collected by corporations and governments, with avowed goals that relate data analysis to social welfare, it is timely to pursue formal treatments of sequential information design, to understand how to strategically inform the (sequential) decision makers (e.g., users, clients or citizens) impacted by centralized data analysis. In particular, we wish to understand the resulting equilibrium outcomes of both parties. 
As a concrete example, consider an online shopping platform which may make use of learning tools such as reinforcement learning or online convex optimization to manage inventory and ensure profitability~\cite{giannoccaro2002inventory, meisheri2020using}. The platform cannot single-handedly manage its inventory, instead it requires information design (a.k.a., Bayesian persuasion) in its interactions with its suppliers and consumers. On the supply side, it could strategically reveal aspects of consumer sentiment (e.g., rough number of visits, search) to the suppliers in order to guide their sales expectation and negotiate for lower unit prices. On the demand side, it could tactically control displayed product information (e.g., last five remaining, editor's choice) so as to influence consumers' perception of products and consequently their purchase decisions.
Similar situations can be anticipated for a recommendation platform. On the one hand, it should recommend most relevant items to its users for click-through and engagement. On the other hand, its recommendations are subject to misalignments with long-term objectives such as profits (e.g., from paid promotion), social impact (e.g., to prevent misinformation and filter bubbles) or development of a creator ecosystem~\cite{tang2016should, xiao2019beyond, milano2020recommender}.

\subsection{Our Results and Contributions}
To provide a formal foundation for the study of sequential information design, we introduce the \emph{Markov persuasion process} (MPP), where a sender, with informational advantage, seeks to persuade a stream of myopic receivers to take actions that maximize the sender's cumulative utility in a finite-horizon Markovian environment with varying prior and utility functions. We need to address a key challenge regarding the planning problem in MPPs; specifically, how to find persuasive signaling policies that are also optimized for the sender's long-term objective. Moreover, in face of the uncertainty for both the environment and receivers, there is a dilemma that the optimal policy based on estimated prior is not necessarily persuasive and thus cannot induce the desired trajectory, whereas a full information revelation policy is always persuasive but usually leads to suboptimal cumulative utility. So the reinforcement learning algorithm in MPPs has to ensure optimality under the premise of robust persuasiveness. This makes our algorithm design non-trivial and regret analysis highly challenging.

We show how to surmount these analysis and design challenges, and present a no-regret learning algorithm, which we refer to as \fullmodel{} (\ORAI), that provably achieves a  $\tilde{O}\big( \sqrt{d_\phi^2 d_\psi^3H^4T}\big)$ regret with high probability, where $d_\phi, d_\psi$ are dimensions of the feature spaces, $H$ is the horizon length in each episode, $T$ is the number of episodes, and $\tilde {O}(\cdot) $ hides logarithmic factors as well as problem-dependent parameters. 
To establish this result, in Section \ref{sec:optimal-policy} we start by constructing a modified formulation of the Bellman equation that can efficiently determine the optimal (resp. $\epsilon$-optimal) policy with finite (resp. infinite) states and outcomes. Section \ref{sec:op4} then considers the learning problem, in particular the design of the \ORAI{} that adopts both the optimistic principle in utility estimation to incentivize exploration and the pessimism principle in prior estimation to prevent a detrimental equilibrium for the receiver. In Sections \ref{sec:tabular} and \ref{sec:contextual}, we showcase \ORAI{} in the tabular MPPs and contextual Bayesian persuasion problem, respectively, both of which are practical special cases of MPPs. In Section \ref{sec:mdp}, we then generalize these positive results to MPPs with large outcome and state spaces via linear function approximation and generalized linear models.  

In summary, our contributions are threefold.  At the conceptual level, we identify the need for sequential information design in real-world problems and accordingly formulate a novel model, the MPP, to capture the misaligned incentives between the (sequential) decision makers and information possessors. At the methodological level, our key insight is a new algorithmic principle---optimism to encourage exploration and pessimism to induce robust equilibrium behavior. Finally, at the technical level, we develop a novel regret decomposition tailored to this combination of optimism and pessimism in the design of online learning algorithms. The fact that the combined optimism-pessimism concept can still lead to   $O(\sqrt{T})$ regret for strategic setups was not clear before our new regret decomposition lemma.  We expect this design principle and our proof techniques can be useful for other strategic learning problems.

\subsection{Related Work}

Our work is built on the foundation of information design and reinforcement learning. We refer the readers to Section \ref{sec:classic-info-design} and \ref{sec:classic-MDP} for background and formal introductions. Here we focus on the technical and modeling comparisons with related work from dynamic Bayesian persuasion and efficient reinforcement learning. 

\vspace{1mm}
\textbf{Dynamic Bayesian persuasion.}
Starting from seminal work by \citet{kamenica2011bayesian}, the study of Bayesian persuasion looks at the  design problem to influence an uninformed decision maker through strategic information revelation.
Many variants of this model have been studied, with applications in security, advertising, finance, etc.~\cite{rabinovich2015information, xu2015exploring, goldstein2018stress, badanidiyuru2018targeting}. 
More recently, several dynamic Bayesian persuasion frameworks have been proposed to model the long-term interest of the sender. Many papers \cite{ely2017beeps,renault2017optimal,farhadi2021dynamic, lehrer2021markovian} consider the setting where the sender observes the evolving states of a Markov chain, seeks to influence the receiver's belief of the state through signaling and thereby persuade him to take certain actions. In contrast to our setting, at each round, the receiver's action has no influence on the evolution of the Markov process and thus can only maximizes his utility on his belief of current state, given all the historical signals received from the sender. 
In \citet{ely2017beeps, farhadi2021dynamic}, the Markov chain has two states (one is absorbing): the receiver is interested in detecting the jump to the absorbing state, whereas the sender seeks to prolong the time to detection of such a jump. 
\citet{renault2017optimal} shows a greedy disclosure policy that ignores its influence to the future utility can be optimal in Markov chain with special utility functions.
\citet{lehrer2021markovian} characterize optimal strategies under different discount factors as well as the optimal values the sender could achieve. 
Closer to our model is that of \citet{gan2021bayesian}---we both assume the Markov environment with state transition influenced by receiver's action, as well as a separate persuasion state drawn from a prior independent of receiver's action.  However, \citet{gan2021bayesian} focus on the planning problem for the infinite-horizon MDP, solving sender's optimal signaling policy when the environment is known in cases when the receiver is myoptic or far-sighted. In particular, it is shown as NP-hard to approximate an optimal policy against a far-sighted receiver, which also justifies our interest on the myoptic receiver.
Another closely related work \cite{10.1145/3465456.3467593} studies the learning problem in repeated persuasion setting (without Markov state transition) between a stream of myopic receivers and a sender without initial knowledge of the prior. It introduces the notion of regret as well as the robustness principle to this learning problem that we adopt and generalize to our model.

\vspace{1mm}
\textbf{Bayesian Incentive-Compatible Bandit Exploration.}
Our work is also loosely related to a seminal result by~\citet{mansour2021bayesian}, who model the misaligned incentives between a system (i.e., sender) and a stream of myopic agents (i.e., receivers). \citet{mansour2021bayesian} shows that using information asymmetry, the system can create intrinsic incentives for agents to follow its recommendations. In this problem, the sender's objective is limited to the social welfare, i.e, the cumulative utility of all agents, whereas we make no assumption on the sender's utility function and thus her long-term objective. Besides our model is designed to capture more general situations where each receiver could have different priors and utility functions, and the environment might be Markovian with dynamics under the influence of the receivers' actions. 

\vspace{1mm}
\textbf{Efficient Reinforcement Learning.}
Reinforcement learning has seen its successful applications in various domains, such as robotics, finance and dialogue systems~\cite{kober2013reinforcement, zheng2020ai, li2016deep}.
Along with the empirical success, we have seen a growing quest to establish provably efficient RL methods. 
Classical sample efficiency results focus on tabular environments with small, finite state spaces~\cite{auer2008near, osband2016generalization, azar2017minimax, dann2017unifying, strehl2006pac, jin2018q, russo2019worst}.  Notably, through the design principle, known as optimism in the face of uncertainty~\cite{lattimore2020bandit}, an RL algorithm would provably incur a $\Omega(\sqrt{|\cS||\cA|T})$ regret under the tabular setting, where $\cS$ and $\cA$ are the state and action spaces respectively~\cite{jin2018q, azar2017minimax}. 
More recently, there have been advances in RL with function approximation, especially the linear case. \citet{jin2020provably} proposed an efficient algorithm for a setting where the transition kernel and the utility function are both linear functions with respect to a feature mapping: $\phi: \cS \times \cA \to \RR^d$. A similar assumption has been studied for different settings and has led to sample efficiency results~\cite{yang2019sample, du2019good, neu2020unifying, zanette2020learning, he2021logarithmic}.  Moreover, other general function approximations have been studied in parallel, including generalized linear function approximation~\cite{wang2019optimism}, linear mixture MDPs based on a ternary feature mapping~\cite{ayoub2020model, zhou2021provably, cai2020provably, zhou2021nearly}, kernel approximation~\cite{yang2020approximation} as well as models based on the low Bellman rank assumption~\cite{jiang2017contextual, dann2018oracle}. We make use of these function approximation techniques to model our conditional prior, and we show how to integrate the persuasion structure into these efficient reinforcement learning frameworks, thereby obtaining sample efficient result for large-scale MPPs.

\section{Preliminaries}
This section provides some necessary background in information design and Markov decision processes, as preparation for our model of Markov persuasion processes presented in the next section. 
\subsection{Basics of Information Design} 
\label{sec:classic-info-design}    
Classic information design \cite{kamenica2011bayesian} considers the persuasion problem between a single sender (she) and  receiver (he).  
The receiver is the only actor, and looks to take an action $a \in \cA$ which results in receiver utility    $u(\omega,a)$ and   sender   utility $v(\omega,a)$. Here $\omega\in \Omega$ is the realized \emph{outcome} of certain environment uncertainty, which is drawn  from a prior distribution $\mu \in \Delta(\Omega)$, and $\cA$ is a finite set of available actions for the receiver.  While $u, v: \Omega\times \cA \to [0, 1]$ and  the prior distribution $\mu$ are all common knowledge,  the sender possesses an informational advantage and can privately observe the realized outcome $\omega$.  The persuasion problem studies how the sender can selectively reveal her private information about $\omega$ to influence the receiver's decisions and ultimately maximize her own expected utility $v$. 
To model the sender's strategic revelation of information, it is standard to use a \emph{signaling scheme}, which essentially specifies the conditional distribution of a random variable (namely the \emph{signal}), given the  outcome $\omega$.  Before the  realization of the outcome, the sender commits to such a signaling scheme.  
Given   the realized outcome, the sender samples a \emph{signal} from the conditional distribution according to the \emph{signaling scheme} and reveals it to the receiver. 
Upon receiving this \emph{signal}, the receiver infers a posterior belief about the outcome via Bayes' theorem (based on the correlation between the signal and outcome $\omega$ as promised by the signaling scheme) and then chooses an action $a$ that maximizes the expected utility. %
A standard revelation-principle-style argument shows that it is without loss of generality to focus on \emph{direct} and \emph{persuasive} signaling schemes \cite{kamenica2011bayesian}. A scheme is direct if each signal corresponds to an action recommendation to the receiver, and is persuasive if the recommended action  indeed maximizes the receiver's a posteriori expected utility. 
More formally, in a direct signaling scheme , $\pi = (\pi(a | \omega): \omega \in \Omega, a\in \cA)$, $\pi(a | \omega)$ 
denotes the probability of recommending action $a$ given realized outcome $\omega$. Upon receiving an action recommendation $a$, the receiver computes a posterior belief for $\omega$: $\Pr(\omega|a) = \frac{\mu(\omega)\pi(a |\omega)}{\sum_{\omega'} \mu(\omega') \pi(a | \omega')} $. Thus, the action recommendation $a$ is persuasive if and only if $a$ maximizes the expected utility w.r.t. the posterior belief about $\omega$; i.e.,
$\sum_{\omega} \Pr(\omega|a) \cdot u(\omega, a) \geq \sum_{\omega} \Pr(\omega|a) \cdot u(\omega, a')  $
for any $a'\in \cA$.   Equivalently, we define \emph{persuasiveness} as 
\begin{equation*}
  \text{Persuasiveness: } \quad   \sum_{\omega \in \Omega}  \mu(\omega) \pi(a|\omega)  \cdot [u(\omega, a) - u(\omega, a')] \geq 0, \forall a, a' \in \cA.
\end{equation*}
Let $\cP = \{\pi : \pi(\cdot | \omega) \in \Delta(\cA) \text{ for each } \omega \in \Omega\}$ denote  the set of all signaling schemes.
To emphasize that  the definition  of persuasiveness depends on the prior $\mu$, we denote the set of  persuasive schemes on prior $\mu$~by 
\begin{equation*}
\Ps(\mu) \coloneqq \left\{\pi \in \cP: \sum_{\omega \in \Omega} \mu(\omega) \pi(a | \omega)\left[ u(\omega, a)-u\left(\omega, a^{\prime}\right)\right] \geq 0,\quad \forall a,a'\in \cA \right\}   . 
\end{equation*}
Given a persuasive signaling scheme $\pi \in \Ps(\mu)$, it is in the receiver's best interest to take the recommended action and thus  the sender’s expected utility becomes 
$ V(\mu, \pi) \coloneqq \sum_{\omega \in \Omega} \sum_{a \in \cA} \mu(\omega) \pi(a | \omega) v(\omega, a) .$

Therefore, given full knowledge of the persuasion instance, the sender can solve for an optimal persuasive signaling scheme that maximizes her expected utility through the following linear program (LP) which searches for a persuasive signaling scheme that maximizes $    V(\mu, \pi)$ (see, e.g., \cite{dughmi2019algorithmic} for details): 
\begin{equation*}
   \text{Persuasion as an LP:}   \qquad  \OPT\left(\mu \right) \coloneqq  \max_{\pi \in \Ps(\mu)} \quad   V(\mu, \pi).
\end{equation*}

\subsection{Basics of Reinforcement Learning and Markov Decision Processes}
\label{sec:classic-MDP}    
The Markov decision process (MDP) \cite{puterman2014markov, sutton2018reinforcement} is a classic mathematical framework for the sequential decision making problem. 
In this work, we focus on the model of episodic MDP. 
Specifically, at the beginning of the episode, the environment has an initial state $s_1$ (possibly picked by an adversary). Then, at each step $h \geq 1$, the agent takes some action $a_h \in \cA$ to interact with environment at state $s_h \in \cS$. The state $s_h$ obeys a Markov property and thus captures all relevant information in the history $\{s_i\}_{i<h}$.
 Accordingly, 
 the agent  receives the utility $v_h(s_h,a_h) \in  [0,1]$ and the system evolves to  the state of  the next step $s_{h+1}\sim P_h(\cdot | s_h, a_h) $.
 Such a process terminates after $h  = H$, where $H $ is also known as the horizon length.  
 Here,  $\cA$ is a finite set of available actions for the agent, $\cS$ is the (possibly infinite) set of MDP states. 
 The utility function $v_h: \cS\times\cA \to [0,1]$ and  transition kernel $P_h: \cS\times\cA \to  \Delta (\cS) $ may vary at each step. 
 A policy of the agent $\pi_h: \cS \to \Delta (\cA)  $ characterizes her decision making process at step $h$---after observing the state $s$, the agent takes action $a$ with probability $\pi_h(a |s)$. 

In an episodic MDP with $H$ steps, under policy $\bpi = \{ \pi_h \}_{h\in [H]}$, we define the value function as the expected value of cumulative utilities starting from an arbitrary state, 
$$ 
V_h^{\pi}(s) \coloneqq \EE_{P,\bpi} \bigg[ \sum_{h'=h}^{H} v_h(s_{h'}, a_{h'}) \bigg| s_{h'} = s \bigg],\quad \forall s\in \cS, h\in [H].
$$
Here $\EE_{P,\bpi}$ means that the expectation is taken with respect to the trajectory $\{s_h , a_h \}_{h\in [H]}$, which is generated by policy $\bpi$ on the transition model $P = \{ P_h \}_{h\in [H]}$.
Similarly, we define the action-value function as the expected value of cumulative utilities starting from an arbitrary state-action pair,
$$ 
Q_h^{\pi}(s, a) \coloneqq v_h(s_h, a_h) + \EE_{P,\bpi} \bigg[ \sum_{h'=h+1}^{H} v_h(s_{h'}, a_{h'}) \bigg| s_{h'} = s, a_{h'} =a \bigg], \quad \forall s\in \cS, a\in \cA, h\in [H]. 
$$

The optimal policy is defined as $\bpi^* \coloneqq \arg\max_{\bpi} V_h^{\bpi}(s_1) $, which maximizes the (expected) cumulative utility. Since the agent's action affects both immediate utility and next states that influences its future utility, it thus demands careful planning to maximize total utility. 
Notably, $\bpi^*$ can solved by dynamic programming based on the Bellman equation~\cite{Bel}.  
Specifically, with $V_{H+1}^*(s)=0$ and for each $h$ from $H$ to $1$, iteratively update
$
    Q_h^*(s, a) =  v_h(s, a) + \EE_{s'\sim P(\cdot |s, a)} V_{h+1}^*(s', a),~
    V_h^*(s)= \max_{a\in\cA} Q_h^*(s, a),
$
and determine the optimal policy $\bpi^*$ as the greedy policy with respect to $\{ Q_h^*\}_{h\in [H]}$.
In online reinforcement learning,
the agent has no prior knowledge of the environment, namely, $\{ v_h, P_h\}_{h\in [H]}$,   but aims to  learn  the optimal policy by interacting with the environment for $T$ episodes. 
For each $t \in [T]$, at the beginning of the $t$-th episode, after observing the initial state $s_1^t$, the agent chooses a policy $\bpi^t$ based on the observations before $t$-th episode. The discrepancy between $V^{\bpi^t}_1(s_1^t)$ and $V^*_1(s_1^t)$ serves as the suboptimality of the agent at the $t$-th episode. The performance  of the online learning algorithm is measured by the expected regret, $\Reg(T) \coloneqq \sum_{t=1}^{T} [V^*_1(s_1^t) - V^{\bpi^t}_1(s_1^t) ].$

\section{Markov Persuasion Processes} 
This section introduces the Markov Persuasion Process (MPP), a novel model for sequential information design in \emph{Markovian environments}. 
It notably captures the motivating yet intricate real-world problems in Section \ref{sec:motivating-exp}. Furthermore, our MPP model is readily applicable to generalized settings with large state spaces by incorporating function approximation techniques.

\subsection{A Model of Markov Persuasion Processes (MPPs)}\label{sec:model:basic}
We start by abstracting the sequential information design problem instances in Section \ref{sec:motivating-exp} into MPPs. 
Taking as an example recommendation platform for ad keywords, we view the platform as the \emph{sender}, the advertisers as the \emph{receivers}. The advertisers decide the \emph{actions} $a \in \cA$ such as whether to accept the recommended keyword. To better reflect the nature of reality, we model two types of information for MPPs, \emph{outcome} and \emph{state}. 
We use the notion of \emph{outcome} $\omega\in \Omega$ to characterize the sender's private information in face of each receiver, such as the features of searchers for some keyword. The \emph{outcome} follows a prior distribution such as the general demographics of Internet users on the platform. The platform can thus leverage such fine-grained knowledge on keyword features, matching with the specific ad features of each advertiser, to persuade the advertisers to take a recommendation of keywords.
Meanwhile, we use the notion of \emph{state} $s\in \cS$ to characterize the Markovian state of the environment, e.g., the availability of ad keyword slots. The state is affected by the receiver's action, as the availability changes after some keywords get brought.\footnote{\label{fn:examples}Similarly, we can view the online shopping platform as the sender who persuades a stream of {receivers} (supplier, consumer) to take certain {action}, whether to take an offer or make a purchase. In this case, sender can privately observe the outcomes such as the consumer sentiments on some random products based on the search and click logs, whereas the states are product reviews, sales or shipping time commonly known to the public and affected by the actions of both supply and demand sides.
In case of rider-sharing, outcome represents the fine-grained knowledge of currently active rider types that are privately known to the platform and generally stochastic in accordance to some user demographics, whereas the state captures the general driver supply or rider demand at locations that is affected by the drivers' decisions.
}
Naturally, both sender's and receiver's utility are determined by the receiver's action $a$ jointly with the state of environment $s$ and realized outcome $\omega$, i.e., $u, v: \cS \times \Omega \times \cA \to [0,1]$.
Meanwhile, as these applications could serve thousands or millions of receivers every day, to reduce the complexity of our model we assume each receiver appears only once and thus will \emph{myopically} maximizes his utility at that particular step, whereas the sender is a system planner who aim to maximizes her long-term accumulated expected utility. 

More specifically, an MPP is built on top of a standard episodic  MDP  with \emph{state} space $\cS$, action space $\cA$, and transition kernel $P$.  In this paper, we restrict our attention to  \emph{finite-horizon} (i.e., episodic) MPPs with $H$ steps denoted by $[H] = \{ 1, \cdots, H \}$, and leave the study of infinite-horizon MPPs as an interesting future direction. 
At a high level, there are two major differences between  MPPs and MDPs. First, in a MPP, the planner \emph{cannot} directly take an action but instead can leverage its information advantage and ``persuade'' a receiver to take a desired action $a_h$ at each step $h \in [H]$. Second, in an  MPP, the state transition  is affected not only by the current action $a_h$ and state $s_h$, but also by the realized outcome $\omega_h$ of Nature's probability distribution.
Specifically, the state transition kernel at step $h$ is denoted as $P_h(s_{h+1}| s_h, \omega_h, a_h)$. To capture the sender's persuasion of a receiver to take actions at step $h$, we assume that a fresh receiver arrives at time $h$ with a prior $\mu_h$ over the outcome $\omega_h$. The planner, who is the \emph{sender} here, can observe the realized outcome $\omega_h$ and would like to strategically reveal information about $\omega_h$ in order to persuade the receiver to take a certain action $a_h$. %

Differing from classical single-shot information design, the immediate utility functions $u_h, v_h$ for the receiver and sender  vary not only at each step $h$ but also additionally depend on the commonly observed state $s_h$ of the environment. 
We assume the receiver to have full knowledge of his utility $u_h$ and prior $\mu_h$ at each step $h$, and would take the recommended action $a_h$ if and only if $a_h$ maximizes his expected utility under the posterior for $\omega_h$.\footnote{This assumption is not essential but just for technical rigor. Because even if receivers have limited knowledge or computational power to accurately determine the utility-maximizing actions, the sender should have sufficient ethical or legal reasons to comply with the persuasive constraints in practice. And the receivers would only take the recommendation if the platform has good reputation (i.e., persuasive with high probability). } 

Formally, an  MPP with a horizon length $H$   proceeds as follows at each step $h \in [H]$:
\vspace{-10pt}
\begin{center}
\fcolorbox{black}{gray!20!white}{
\parbox{0.995\textwidth}{	

\begin{enumerate}
\item A fresh receiver with prior distribution $\mu_h \in \Delta(\Omega)$ and utility $u_h: \cS \times \Omega \times \cA \to [0,1]$ arrives. 

\item The sender commits to a \emph{persuasive} signaling policy $\pi_h: \cS \to \cP$, which is public knowledge. 

\item After observing the realized state $s_h$  and outcome $\omega_h$, the sender accordingly recommends the receiver to take an action $a_h \sim \pi_h (\cdot|s_h,\omega_h)$.

\item Given the recommended action $a_h$, the receiver takes an action $a'_h$,  receives utility $u_h(s_h , \omega_h ,a'_h)$ and then leaves the system. Meanwhile, the sender receives utility $v_h(s_h, \omega_h ,a'_h)$. 
\item The next state $s_{h+1}  \sim P_{h} (\cdot | s_h, \omega_h, a'_h)$  is generated   according to   $P_{h}: \cS \times \Omega \times \cA \to \Delta(\cS)$, the \mstate{} transition kernel at the $h$-th step. 
\end{enumerate}
}}
\end{center}

Here we coin the notion of a \emph{signaling policy} $\pi_h$ as a mapping from \mstate{} to a signaling scheme at the $h$-th step. It captures a possibly multi-step procedure in which the sender  commits to a signaling scheme after observing the realized state and then samples a signal after observing the realized outcome.
For notational convenience, we denote $\pi(a|s,\omega)$ as the probability of recommending action $a$ given \mstate{} $s$ and realized \pstate{} $\omega$. 
We   can also generalize the notion of persuasiveness from classic information design to MPPs.  
Specifically,  we  define $\Ps(\mu, u)$ as the persuasive set that contains all signaling policies that are persuasive to the receiver with utility $u$ and prior $\mu$ for all possible state $s \in \cS$:
\begin{align*}
\Ps(\mu, u) \coloneqq &  \bigg\{\pi: \cS \to \cP:  \notag \\
&  \qquad \qquad \int_{\omega \in \Omega} \mu(\omega)\pi(a|s,\omega) \big[ u(s,\omega,a)-u(s,\omega,a') \big]\ud\omega\ge 0,\quad \forall a,a'\in \cA, s\in \cS \bigg\}.
\end{align*}
Recall that $\cP$ consists of all mappings from $\Omega$ to $\Delta(\cA)$.
As such, the sender's persuasive signaling scheme $\pi_h \in \Ps(\mu_h, u_h)$ is essentially a stochastic policy as defined in standard MDPs---$\pi_h$ maps a state $s_h$ to a stochastic action $a_h$---except that here the probability of suggesting action $a_h$ by policy $\pi_h$ depends additionally on the realized outcome $\omega_h$ that is only known to the sender.

We say $\bpi \coloneqq \{\pi_h\}_{h\in [H]}$ is a \emph{feasible} policy of the MPP if $\pi_h \in \Ps(\mu_h, u_h), \forall h\in [H]$, because the state transition trajectory would otherwise be infeasible if the receiver is not guaranteed to take the recommended action, i.e., $a'_h \neq a_h$. We denote the set of all \emph{feasible} policies as $\cP^H \coloneqq \prod_{h\in [H]}\Ps(\mu_h, u_h)$.

\subsection{ MPPs: the Generalized Version with Contexts and Linear Parameterization}\label{sec:model:general}

To provide a broadly useful modeling concept, we also study a  generalized setting  of the Markov Persuasion Process with contextual prior and a possibly large space of states, outcomes and contexts.

\paragraph{Contextual Prior.} At the beginning of each episode, a sequence of contexts $C=\{c_h \in \cC \}_{h\in[H]}$ is realized by Nature  and becomes public knowledge. And we allow the prior $\mu_h$ to be influenced by the context $c_h$ at each step $h$, and thus denote it by   $\mu_h(\cdot | c_h)$. 
Specifically, the contextual information is able to model the uncertainty such as the varying demographics of active user group affected by events (e.g., scheduled concerts or sport games in ride-sharing) at different time of the day.\footnote{In the case of the online shopping platform, the prior of consumer interests may be affected by the different holidays or seasons at different time of year.} 
Here we   allow  the sequence of contexts to be  adversarially generated. 

\paragraph{Linear Parameterization.} We also relax the state, context and outcome space $\cS, \cC, \Omega$ to be continuous and additionally assume that the transition kernels and utility functions are linear, and the conditional priors of outcomes are generalized linear models (GLM) of the context at each steps.
More formally, for each step $h\in[H]$, our linearity condition assumes:
\begin{itemize}
	\item The sender's utility is $v_h \coloneqq v_h^*(s_h, \omega_h,a_h) = \psi(s_h, \omega_h, a_h)^\top \gamma_h^*$, where (1)  $\psi(\cdot,\cdot,\cdot) \in \RR^{d_\psi}$ is a known feature vector; (2) $\gamma_{h}^* \in  \RR^{d_\psi}$ is the   unknown  linear  parameter at step $h$. 
	\item The next state $s_{h+1}$ is drawn from the distribution $P_{M,h}(\cdot | s_h, \omega_h, a_h) = \psi(s_h,\omega_h, a_h)^\top M_h(\cdot)$, 
	where $M_h = (M_h^{(1)},M_h^{(2)}, \dots, M_h^{(d_\psi)} ) $ is a vector of $d_\psi$  unknown  measures over $\cS$ at step $h$. 
	\item The \pstate{} $\omega_h \in\RR$ subjects to a generalized linear model (GLM), which models a wider range of hypothesis function classes.\footnote{We note that GLM is a strictly generalization of the linear model assumption that we have for the distribution of transition kernel $P$. While we could use similar technique to extend the distribution of $P$ to GLM using techniques similar to that  in \citet{wang2019optimism}, but we save such an extension  for simplicity, since it is not the primary focus of our work.} 
    Given the \cstate{} $c_h$, there exists a link function $f:\RR\to\RR$ such that $\omega_h=f(\phi(c_h)^\top\theta^*_h )+z_h$, where $\phi(\cdot)\in\RR^{d_\phi}$ is a feature vector and $\theta_h^*\in\RR^{d_\phi}$ is an unknown parameter. The noises $\{z_h\}_{h\in[H]}$ are independent $\sigma$-sub-Gaussian variables with zero mean.
	We denote the prior of $\omega_h$ with parameter $\theta$ at \cstate{} $c$ as $\mu_{\theta}(\cdot \vert c)$.  
\end{itemize}

Without loss of generality, we assume that there exist $\Phi, \Psi$ such that $\Vert \phi(s)\Vert \leq \Phi$,\footnote{For the simplicity of notation, we will omit the subscript of the norm whenever it is an $L_2$ norm in this paper.}  $\Vert \psi(s,\omega,a) \Vert \leq 
\Psi$ 
for all $s \in \cS, \omega \in \Omega$ and $a \in \cA$. We also assume that $\Vert\theta_h^*\Vert\leq L_\theta$, $\Vert\gamma_h^*\Vert\leq L_\gamma$, $\Vert M_h^* \Vert \leq L_M$, $|\mathcal{A}| \geq 2$, $|\Omega| \geq 2$. 
Such a regularity condition is common in the RL literature. 

\subsection{Optimal Signaling Policy in MPPs}
\label{sec:optimal-policy}

In order to maximize the sender's utility, we study the optimal policy in MPPs, in analogy to that of standard MDPs. We start by considering the value of  any feasible policy $\bpi$. For each step $h \in [H]$, we   define the value function for the sender $V_h^\pi:\cS \to \mathbb{R}$ as the expected value of cumulative utilities under $\bpi$ when starting from an arbitrary state at the $h$-th step. That is, for any $s \in \cS, h \in [H]$, we define 
\begin{equation*}
    V_h^\pi(s) \coloneqq \mathbb{E}_{P,\mu,\pi} \bigg[ \sum_{h'=h}^H v_{h'}\big(s_{h'},\omega_{h'}, a_{h'} \big) \bigg| s_h=s \bigg],
\end{equation*}
where the expectation $\mathbb{E}_{P,\mu,\pi}$ is taken with respect to the randomness of the trajectory (i.e.,  randomness of \mstate{} transition), realized \pstate{} and the stochasticity of $\bpi$. 
Accordingly, we   define the $Q$-function (action-value function) $Q_h^\pi:\cS \times \Omega \times \cA \to \mathbb{R}$ which gives the expected value of cumulative utilities when starting from an arbitrary state-action pair at the $h$-step following the signaling policy $\bpi$, that is,
\begin{equation*}
    Q_h^\pi(s,\omega,a)\coloneqq v_h(s,\omega,a)+\mathbb{E}_{P,\mu,\pi}\bigg[  \sum_{h'=h+1}^H v_{h'}\big(s_{h'},\omega_{h'}, a_{h'} \big) \bigg| s_h=s,\omega_h=\omega,a_h=a \bigg].
\end{equation*}
By definition, $Q_h(\cdot, \cdot, \cdot), V_h(\cdot) \in [0, h]$, since $v_h(\cdot, \cdot, \cdot) \in [0,1]$.
To simplify notation, for any $Q$-function $Q_h$  and any distributions $\mu_h$ and $\pi_h$  over $\Omega$ and $\cA$,  we additionally denote 
\begin{equation*}
\begin{split}
    \big\langle Q_h,\mu_h\otimes\pi_h\big\rangle_{\Omega\times\cA}(s) &\coloneqq  \mathbb{E}_{\omega\sim\mu_h,a\sim \pi_h(\cdot|s,\omega)}\left[Q_h(s,\omega,a)\right].
\end{split}
\end{equation*}
Using this notation, the Bellman equation associated with signaling policy $\bpi$ becomes
\begin{equation}\label{eq:policy-Bellman}
    Q_h^\pi(s,\omega,a) =  (v_h+P_h V_{h+1}^\pi)(s,\omega,a),~~
    V_h^\pi(s)= \big\langle Q_h^\pi,\mu_h\otimes\pi_h\big\rangle_{\Omega\times\cA}(s),~~
    V_{H+1}^\pi(s)=0,
\end{equation}
which holds for all $s \in \mathcal{S}, \omega \in \Omega, a \in \mathcal{A}$. Similarly, the Bellman optimality equation is
\begin{equation}\label{eq:optimal-Bellman}
    Q_h^*(s,\omega,a) =  (v_h+P_h V _{h+1}^*)(s,\omega,a),~
        V_h^*(s)= \max_{\pi_h' \in \Ps(\mu_h,u_h)}\big\langle Q_h^*,\mu_h\otimes\pi_h' \big\rangle_{\Omega\times\cA}(s),~
    V_{H+1}^*(s)=0.
\end{equation}

We remark that the above equations implicitly assume the context $C=\{ c_h \}_{h\in[H]}$ (and thus the priors) are determined in advance. To emphasize the values' dependence on context which will be useful for the analysis of later learning algorithms,  we extend the notation to $V^{\pi}_h(s; C), Q_h^\pi(s,\omega,a; C)$ to specify that the value (resp. Q) function is estimated based on which prior $\mu$ conditioned on which sequence of context $C$. 

\paragraph{A Note on Computational Efficiency. } We note that the above Bellman Optimality Equation  in \eqref{eq:optimal-Bellman} also implies an efficient dynamic program to  compute the optimal policy $\bpi^*$ in the basic tabular model of MPP in Subsection \ref{sec:model:basic}, i.e., when $s \in \mathcal{S}, \omega \in \Omega, a \in \mathcal{A}$ are all discrete. This is because the maximization problem in equation \eqref{eq:optimal-Bellman} can be solved efficiently be a linear program.  The generalized MPP of subsection \ref{sec:model:general} imposes some computational challenge due to infinitely many outcomes and states.
Fortunately, it is already known that planning in the infinite state MDP with linear function approximation can also be solved efficiently \cite{jin2020provably}. Following a similar analysis, we can determine $ Q_h^*(\cdot,\cdot, \cdot) $ through a linear function of $q_h^*\in \RR^{d_{\psi}}$ with the observed feature $\psi(\cdot,\cdot, \cdot)$. Hence, the dominating operation is to compute $\max_{\pi \in \Ps(\mu_h,u_h)} \langle Q_h^*,\mu_h\otimes\pi_h   \rangle_{\Omega\times\cA}(s)$ at each step.
Let the sender utility function be $Q_h^*$; such an optimization is exactly the problem of optimal information design with infinitely many outcomes but finitely many actions, which has been studied in previous work \cite{dughmi2019algorithmic}. It turns out that there is an efficient algorithm that can signal on the fly for any given outcome $\omega$ and obtains an $\epsilon$-optimal persuasive signaling scheme in $\poly( 1 / \epsilon)$ time. Therefore, in our later studies of learning, we will take these algorithms as given and simply assume that we can compute the optimal signaling scheme efficiently at any given state $s$. One caveat is that our regret guarantee will additionally lose an additive $\epsilon$ factor at each step  due to the availability of only an $\epsilon$-optimal algorithm, but this loss can be negligible when we set $\epsilon = O(1/ (TH))$ by using a $\poly(TH)$ time algorithm.  
\section{Reinforcement Learning in MPPs and the Optimism-Pessimism Principle}

In this section, we study online  reinforcement learning (RL) for learning the optimal signaling policy on an   MPP. 
Here the learner only knows the utility functions of the receivers\footnote{\label{fn:receiver-payoff} The receiver's utility is known to the sender because the pricing rules are usually transparent, some are even set by the platform. For example, a rider-sharing platform usually sets per hour or mile payment rules for the drivers. } and has no prior knowledge about the prior distribution, the sender's utility function, and the transition kernel. 
 While the computation of optimal policy in MPPs in Section \ref{sec:optimal-policy} may appear analogous to that of a standard MDP, as we will see that the corresponding RL problem turns out to be significantly different, partially due to the presence of the stream of receivers, whose decisions are \emph{self-interested} and not under the learner's control. This makes the learning challenging because if the receivers' incentives are not carefully addressed, they may take actions that are extremely undesirable to the learner. 
 Such concern leads to the integration of the \emph{pessimism} principle into our learning algorithm design. Specifically, our learner will be optimistic to  the estimation of the  $Q$-function,  similar to  many other RL algorithms, 
 in order to encourage exploration.
 But more interestingly, it will be pessimistic to the uncertainty in the estimation of the prior distributions in order to prepare for detrimental equilibrium behavior. 
Such   dual considerations lead  to an interesting \emph{optimism-pessimism principle} (OPP)   for learning MPPs 
under the online setting.
   From a technical point of view, our main contribution is to prove how the mixture of  optimism and pessimism principle can still lead to no regret algorithms, and this proof crucially hinges on a robust property of the MPP model which we develop and carefully apply to the regret analysis. 
To the best of our knowledge, this is the first time that OPP is employed to learn the optimal information design in an online fashion. We   prove that it can   not only satisfy incentive   constraints but also   guarantees   efficiency in terms of both sample complexity and computational complexity.

In order to convey our key design ideas before diving into the  intricate technicalities, this section singles out two representative special cases of the online sequential information design problem. In a nutshell, we present a learning algorithm \ORAI{} %
that combines the principle of optimism and pessimism such that the sender can learn to persuade without initially knowing her own utility or the prior distribution of outcomes. 
In the \emph{tabular MPP}, we illustrate the unique challenges of learning to persuade arising from the dynamically evolving environment state  according to a Markov process. 
Through the \emph{contextual Bayesian persuasion}, we showcase the techniques necessary for learning to persuade with infinitely many states (i.e., contexts) and outcomes.  We shall omit most proofs in this section to focus on the high-level ideas, because the proof for the general setting presented in Section \ref{sec:mdp} suffices to imply all results for the two special cases here.

\subsection{Learning Optimal Policies in MPPs: Setups and Benchmarks}

We consider the episodic reinforcement learning problem in finite-horizon MPPs. Different from the full knowledge setting in Section \ref{sec:optimal-policy}, the transition kernel, the sender's utility function and the outcome prior at each step of the episode, $\{P_h, v_h, \mu_h \}_{h\in[H]}$, are all unknown. The sender has to learn the optimal signaling policy by interacting with the environment as well as a stream of receivers in $T$ number of episodes. For each $t \in [T] = \{1, \cdots, T \}$, at the beginning of $t$-th episode, given the data $\{( c_h^\tau,s_h^\tau, \omega_h^\tau, a_h^\tau, v_h^\tau) \}_{h\in [H], \tau\in [t-1]} $, 
the adversary picks the context sequence $\{c_h^t \}_{h\in [H]}$ as well as the initial state $s_1^t$,
and the agent accordingly chooses a signaling policy $\bpi^t = \{\pi_h^t \}_{h\in [H]}$. Here  $v_h^\tau$ is the utility collected by the sender  at step $h$ of episode $\tau$. 

\paragraph{Regret}
To evaluate the online learning performance, 
given the ground-truth \pstate{} prior  $\bmu^*=\{ \mu_h^* \}_{h\in[H]}$, we define the sender's total (expected) regret over the all $T$ episodes   as 
\begin{equation}\label{eq:defi-regret}
    \Reg(T,\bmu^*) \coloneqq \sum_{t=1}^{T}\left[ V^*_1(s_1^t;C^t)-V^{\bpi^t}_1(s_1^t; C^t)\right].
\end{equation}
Note that if $\bpi^t$ is not always feasible under $\bmu^*$, but is only persuasive with high probability, so the corresponding regret under $\bpi^t$ should be also in high probability sense. 

It turns out that in certain degenerate  cases  it is impossible to achieve  a sublinear regret. For example, if the set of possible posterior outcome distributions that induce some $a\in\cA$ as the optimal receiver action has zero measure, then such posterior within a zero-measure set can never be exactly induced by a signaling scheme without a precise knowledge of the prior. Thus, the regret could be $\Omega(T)$ if receiver cannot be persuaded to play such action $a$.
Therefore, to guarantee no regret, it is necessary to introduce certain regularity assumption on the MPP instance. Towards that end, we shall assume that the receivers' utility $u$ and prior $\mu$ at any step of the MPP instance always satisfies a minor assumption of $(p_0,D)$-regularity as defined below.  

\paragraph{Regularity Conditions}
An instance satisfies \emph{$(p_0,D)$-regularity}, if for any feasible state $s\in \cS$ and context $c\in \cC$, we have
$$
\PP_{\omega\sim \mu(\cdot |c) }\left[ \omega \in \cW_{s, a}(D) \right] \geq p_0, \quad \forall a\in \cA,
$$
where $\mu$ is the ground-truth prior of outcomes and $\cW_{s, a}(D) \triangleq \{ \omega :  u(s,\omega, a)  - u(s,\omega, a') \geq D, \forall a' \in \cA/\{a\} \} $ is the set of outcomes $\omega$ for which the action $a$ is optimal for the receiver by at least $D$ at state $s$. 
.
In other words, an instance is $(p_0,D)$-regular if every action $a$ has  at least probability $p_0$, under randomness of the outcome, to be strictly better than other actions by at least $D$.  This regularity condition is analogous to a regularity condition of  \citet{10.1145/3465456.3467593} but  is generalizable to infinite outcomes as we consider here. 

\subsection{Algorithm: Optimism-Pessimism Principle for Persuasion Process (OP4) }
\label{sec:op4}
The learning task in MPPs involves two intertwined challenges: (1)
How to persuade the receiver to take desired actions under unknown $\mu_h$? (2) Which action to persuade the receiver to take in order to explore the underlying environment?  
For the  first challenge, due to having finite data, it is impossible to perfectly recover $\mu_h$. 
We can only hope to construct an approximately accurate estimator of $\mu_h$. 
To guard against potentially detrimental equilibrium behavior of the 
receivers due to the prior estimation error, we propose to adopt the pessimism principle. 
Specifically, before each episode, we conduct uncertainty quantification 
for the estimator of the prior distributions, which enables us to construct a confidence region containing the true prior with high probability. 
Then we propose to find the signaling policy within a pessimistic candidate set---signaling policies that are simultaneously 
persuasive with respect to all prior distributions in the confidence region.
When the confidence region is valid, such a pessimism principle ensures that the executed signaling policy is always persuasive with respect to the true prior. 
Furthermore, to address the second challenge, we adopt the 
celebrated principle of optimism in the face of uncertainty \cite{lattimore2020bandit}, 
which has played a key role in the online RL literature. 
The main idea of this principle is that, the uncertainty of the $Q$-function estimates essentially reflects our uncertainty about the underlying model. 
By adding the uncertainty as a bonus function, we encourage actions with high uncertainty to be recommended and thus taken by the receiver when persuasiveness is satisfied. 
We  then fuse the two principles into the \ORAI{} algorithm in Algorithm \ref{alg:abstract}.

\begin{algorithm}[H]
\caption{\ORAI{} Overview}
\label{alg:abstract}
\begin{algorithmic}[1] %
\FOR{episode $t = 1 \dots T$} 
\STATE Receive the initial state $\{s_1^t\}$ and context $C^t = \{ c_h^t\}_{h=1}^{H}$.
\STATE For each step $h\in[H]$, estimate prior $\mu_h^t$ along with the confidence region $\mu_{\cB_h^t}$, and construct an optimistic $Q$-function $Q_h^t$ iteratively with the value function $V_h^t$.
\FOR{step $h=1,\ldots,H$}
\STATE Choose robust signaling scheme $\pi_h^t\in \arg\max_{\pi_h\in\Ps(\mu_{\cB_h^t},u_h)} \big\langle Q_h^t,\mu_h^t\otimes\pi_h\big\rangle_{\Omega\times\cA}(s_h^t; C^t)$.
\STATE Observe state $s_h$, outcome $\omega_h$ and accordingly recommend action $a \sim \pi^t_h(\omega_h, \cdot)$ to the receiver.
\ENDFOR
\ENDFOR
\end{algorithmic}
\end{algorithm}

\paragraph{Pessimism to Induce Robust Equilibrium Behavior}
From the data in the past episode, the sender can estimate the mean of the prior as well as obtain a confidence region through concentration inequalities. Given this partial knowledge of the prior distribution, the sender needs to design a signaling scheme that works in the face of any possible priors in the confidence region in order to ensure the receiver will take its recommended action with high probability. 
Specifically, we let $\uB_{\Sigma}(\theta,\beta)\coloneqq\{\theta':\Vert\theta'-\theta\Vert_{\Sigma}\leq \beta\}$ denote the closed ball in $\norm{\cdot}_{\Sigma}$ norm of radius $\beta>0$ centered at $\theta \in \RR^{d_\theta}$.
For any set $\mathcal{B}\subseteq \RR^{d_\theta}$, we let $\Ps(\mu_\cB,u)$ denote the set of signaling policies that are simultaneously persuasive under all weigh vectors $\theta \in \mathcal{B}$: $\Ps(\mu_\cB,u) \coloneqq \bigcap_{\theta \in \mathcal{B}}\Ps(\mu_\theta,u)$. For any non-empty set $\mathcal{B}$, the set $\Ps(\mu_\cB,u)$ is convex since it is an intersection of convex sets $\Ps(\mu_\theta,u)$, and is non-empty since it must contain the full-information signaling scheme. We note that since $\Ps(\mu_\cB,u)$ is a convex set, we can solve the linear optimization among the policies in $\Ps(\mu_\cB,u)$ in polynomial time (see e.g., \cite{10.1145/3465456.3467593}).

\paragraph{Optimism to Encourage Exploration}
In order to balance exploration and exploitation, we adopt the principle of optimism in face of uncertainty to the value iteration algorithm based on Bellman equation, following in a line of work in online RL such as $Q$-learning with UCB exploration~\cite{jin2018q}, UCBVI~\cite{azar2017minimax}, LSVI-UCB~\cite{jin2020provably} (also see \cite{wang2020provably, yang2020approximation, wang2021provably} and the references therein). The additional UCB bonus on the $Q$-value encourages exploration and has been shown to be a provably efficient online method to improve policies in MDPs. Moreover, this method not only works for the simple tabular setting, but also generalizes to settings with infinite state spaces by exploiting linearity of the $Q$-function and a regularized least-squares program to determine the optimal estimation of $Q$-value. In fact, within our framework, we could obtain efficient learning result in the infinite state space setting through other optimism-based online RL methods and general function approximators, such as  linear mixture MDPs~\cite{ayoub2020model, zhou2021provably, cai2020provably, zhou2021nearly}, or kernel approximation~\cite{yang2020approximation} or bilinear classes~\cite{du2021bilinear}. 

To provide a concrete picture of the learning process, we instantiate the \ORAI{} algorithm in two special cases and showcase  our key ideas and techniques before delving into the more involved analysis of the generalized MPP setting. Nevertheless, we remark that whether the problem instance is tabular or in the form of linear or generalized linear approximations is not essential and not the focus of our study. \ORAI{} itself only relies on two things, i.e., the uncertainty quantification for $Q$-function and prior estimation. So even the model-free RL framework can be replaced by model-based RL, as we can just construct confidence region for the transition models.

\subsection{Warm-up I: Reinforcement Learning in the Tabular MPP}
\label{sec:tabular}

We first consider MPPs in tabular setting with finite states and outcomes, as described in Section \ref{sec:model:basic}. 
In this case, the prior on outcomes at each step degenerates to an unknown but fixed discrete distribution independent of \cstate. 
As linear parameterization is not required for discrete probability distribution, the algorithm can simply update the empirical estimation of $\mu_h^t$ through counting. Similarly, the transition kernel $P_h^*$ is estimated through the occurrence of observed samples, and we uses this estimated transition to compute the $Q$-function $\hat{Q}_h^t$ from the observed utility and estimated value function in the next step, according to the Bellman equation. To be specific, for each  $s\in\cS,\omega\in\Omega,a\in\cA$, $\mu_h^t$ and $\hat{Q}_h^t$ are estimated through the following equations:
\begin{align*}
    \mu_h^t(\omega) &\gets \frac{\lambda/|\Omega|+N_{t,h}(\omega)}{\lambda+t-1},\\
    {\hat Q}_h^t(s,\omega,a) &\gets \frac{1}{\lambda+N_{t,h}(s,\omega,a)}\sum_{\tau\in[t-1]} \big\{\II(s_h^\tau=s,\omega_h^\tau=\omega,a_h^\tau=a)\big[v_h^\tau+V_{h+1}^t(s_{h+1}^\tau) \big]\big\},
\end{align*}
where $N_{t,h}(\omega)=\sum_{\tau\in[t-1]}\II(\omega_h^\tau=\omega)$ and $N_{t,h}(s,\omega,a)=\sum_{\tau\in[t-1]}\II(s_h^\tau=s,\omega_h^t=\omega,a_h^t=a)$ respectively count the effective number of samples that the sender has observed arriving at $\omega$, or the combination $\{s,\omega,a\}$), and $\lambda>0$ is a constant for regularization.

In our learning algorithm, we determine the radius of confidence region $\cB_h^t$ for $\mu_h^t$ according to confidence bound $\epsilon_h^t=O(\sqrt{\log(HT)/t})$. Moreover, we add a UCB bonus term of form $\rho/\sqrt{N_{t,h}(s,\omega,a)}$ to $\hat{Q}_h^t$ to obtain the \emph{optimistic $Q$-function} $Q_h^t$.
Then, it selects a robustly persuasive signaling scheme that maximizes an optimistic estimation of $Q$-function with respect to the current prior estimation $\mu_h^t$. Finally, it makes an action recommendation $a_h^t$ using this signaling scheme, given the state and outcome realization $\{s_h^t,\omega_h^t\}$.

\begin{theorem}\label{thm:contextual_regret}
Let $\epsilon_h^t=\tilde{O}(\sqrt{1/t})$, 
and $\rho= \tilde{O}(|S|\cdot|\Omega|\cdot|A| H)$. Then under $(p_0,D)$-regularity, with probability at least $1-3H^{-1}T^{-1}$, \ORAI{}  has regret of order $\tilde{O}\big( |C|(|S|\cdot|\Omega|\cdot|A|)^{3/2}\cdot H^2 \sqrt{T} / (p_0 D)\big)$ in tabular MPPs.
\end{theorem}

To obtain the regret of \ORAI, we have to consider the regret arising from different procedures. Formal decomposition of the regret is described in Lemma \ref{lm:mdp-decomp}. Separately, we upper bound   errors incurred from estimating $Q$-function (Lemma \ref{lm:mdp-delta}), the randomness of of choosing the outcome, action and next state (Lemma \ref{lm:mdp-zeta}) as well as estimating the prior of outcome and choosing a persuasive signaling scheme that is robustly persuasive for a subset of priors (Lemmas \ref{lm:mdp-gap} and \ref{lm:mdp-concentration}). As the two warm-up models are special cases of the general MPP, the proof of the above properties follows from that of the general MPP setting in Section \ref{sec:proof-sketch}, and thus is  omitted here. %

\subsection{Warm-up II: Reinforcement Learning in Contextual Bayesian Persuasion}
\label{sec:contextual}

We now move to another special case with $H=1$, such that the MPP problem reduces to a contextual-bandit-like problem, where transitions  no longer exist.  
Given a \cstate{} $c$ and a persuasive signaling policy $\pi$, the value function is simply the sender's expected utility for any $s\in\cS$,
\begin{equation*}
V^\pi(s;c)\coloneqq \int_\omega \sum_{a \in A} \mu(\omega|c)\pi(a|s,\omega)v(s,\omega,a) \ud\omega.
\end{equation*}
The sender's optimal expected utility is defined as 
$V^*(s;c) \coloneqq \max_{\pi \in \Ps(\mu(\cdot |c), u)} V^\pi(s; c). $ 

Meanwhile, we consider the general setting where outcome $\omega$ is a continuous random variable that subjects to a generalized linear model. To be specific, the prior $\mu$ is conditioned on the \cstate{} $c$ with the mean value $f(\phi(c)^\top\theta)$. For the prior $\mu$ and link function $f$, we assume the smoothness of the prior and the bounded derivatives of the link function: 
\begin{assumption}\label{assum:prior}
There exists a constant $L_{\mu} > 0$ such that for any parameter $\theta_1,\theta_2$, we have $\big\Vert \mu_{\theta_1}(\cdot\vert c)-\mu_{\theta_2}(\cdot\vert c)\big\Vert_1\leq L_{\mu} \big\Vert f\big(\phi(c)^\top\theta_1\big)-f\big(\phi(c)^\top\theta_2\big)\big\Vert$ for any given context $c$. 
\end{assumption}
\begin{assumption}\label{assum:f}
The link function $f$ is either monotonically increasing or decreasing. Moreover, there exists absolute constants $0<\kappa<K<\infty$ and $0<M<\infty$ such that $\kappa\leq|f'(z)|\leq K$ and $|f''(z)|\leq M$ for all $|z|\leq \Phi L_\theta$.
\end{assumption}

It is natural to assume a Lipschitz property of the distribution in Assumption \ref{assum:prior}. For instance, Gaussian distributions and uniform distributions satisfy this property. Assumption \ref{assum:f} is standard in the literature \cite{filippi2010parametric,wang2019optimism,li2017provably}. Two example link functions are the identity map $f(z)=z$ and the logistic map $f(z)=1/(1+e^{-z})$ with bounded $z$. It is easy to verify that both maps satisfy this  assumption.

Different from the tabular setting, we are now unable to use the counting-based estimator to keep track of the distribution of the possibly infinite states and outcomes. Instead, we resort to function approximation techniques and estimate the linear parameters $\theta^*$ and $\gamma^*$.
In each episode, \ORAI{} respectively updates the estimation and confidence region of $\theta^t$ and $\gamma^t$, with which it can determine the \pstate{} prior under pessimism and sender's utility under optimism. To be specific, the update of $\theta^t$ solves a constrained least-squares problem and the update of $q^t$ solves precisely a regularized one:
\begin{align*}
\theta^t &\gets \arg\min_{\Vert\theta\Vert \leq L_\theta}  \sum_{\tau \in [t-1]} \big[\omega^\tau -f(\phi(c^\tau)^\top\theta_h)\big]^2, \\
\gamma^t &\gets \arg\min_{\gamma\in \RR^\psi} \sum_{\tau \in [t-1]} \norm{ v^\tau-\psi(s^\tau,\omega^\tau,a^\tau)^\top \gamma}^2 + \lambda \norm{\gamma}^2.
\end{align*}
We then estimate the prior by setting $\mu^t(\cdot\vert c)$ to the distribution of $f\big(\phi(c)^\top \theta^t\big)+z$ and estimate the sender's utility by setting $v^t(\cdot,\cdot,\cdot)=\psi(\cdot,\cdot,\cdot)^\top \gamma^t$.
On one hand, to encourage exploration, \ORAI{} adds the UCB bonus term of  form $\rho \norm{ \psi(\cdot,\cdot,\cdot) }_{(\Gamma^t)^{-1}}$ to the $Q$-function, where $\Gamma^t=\lambda I_{d_\psi}+\sum_{\tau \in [t]}\psi(s^\tau,\omega^\tau,a^\tau)\psi(s^\tau,\omega^\tau,a^\tau)^\top$ is the Gram matrix of the regularized least-squares problem and $\rho$ is equivalent to a scalar.  This is a common technique for linear bandits. On the other hand, \ORAI{} determines the confidence region of $\theta^t$ with radius $\beta$, and ensures that signaling scheme is robustly persuasive for any possible (worst case) prior induced by linear parameters $\theta$ in this region. Combining optimism and pessimism, \ORAI{} picks the signaling scheme among the robust persuasive set that maximizes the sender's optimistic utility.

\begin{theorem}\label{thm:contextual_regret}
Under $(p_0,D)$-regularity and Assumption \ref{assum:prior} and \ref{assum:f}, there exists an absolute constant $C_1, C_2>0$ such that, if we set $\lambda=\max\{1,\Psi^2\}$, $\beta=C_1(1+\kappa^{-1}\sqrt{K+M+d_\phi\sigma^2\log(T)})$, and $\rho= C_2 d_\psi \sqrt{\log(4d_\psi \Psi^2 T^3)}$, then with probability at least $1-3T^{-1}$, \ORAI{} has regret of order $\tilde{O}\big( d_\phi\sqrt{d_\psi^3}  \sqrt{T}/ (p_0 D) \big)$ in contextual Bayesian persuasion problems. 
\end{theorem}

Since we estimate the prior by computing an estimator $\theta^t$, we evaluate the persuasiveness of \ORAI{} through the probability that $\theta^*$ lies in the confidence region centered at $\theta^t$ with the radius $\beta=O(\sqrt{d_\phi\log(T)})$ in weighted norm. Due to the smoothness of the prior and the assumption of link function, the error of the estimated prior is bounded by the product of $\beta$ and the weighted norm of feature vector $\Vert \phi(c^t)\Vert_{\Sigma^t}=O(1/\sqrt{t})$, which yields the same conclusion for $\epsilon^t$ in the tabular MPP case. Also compared to \citet{DBLP:journals/corr/LiLZ17a}, we do not require any regularity for $\Sigma^t$, since we add a constant matrix $\Phi^2 I$ to the Gram matrix $\Sigma^t$. This ensures that $\Sigma^t$ is always lower bounded by the constant $\Phi^2>0$. 
The proof of the persuasiveness and sublinear regret of contextual bandit can be viewed as a direct reduction of the MPP case when the total step $H=1$. We decompose the regret in the same way as that in Lemma \ref{lm:mdp-decomp} for MPPs and then estimate the upper bound for each item to measure the regret loss.%

\section{No-Regret Learning in the General Markov Persuasion Process}\label{sec:mdp}

In this section, we present the full version of the \ORAI{} algorithm for MPPs %
and show that it is persuasive with high probability and meanwhile achieves average regret $\tilde{O}\big(d_\phi \cdot d_\psi ^{3/2} H^2 \sqrt{T}/(p_0 D)\big)$.

In the general MPP setting with the linear utility and transition, a crucial property is that the $Q$-functions under any signaling policy is always linear in the feature map $\psi$ (akin to linear MDPs \cite{jin2020provably}).
Therefore, when designing learning algorithms, it suffices to focus on linear $Q$-functions. In our \ORAI{} algorithm, we iteratively fit the optimal $Q$-function, which is parameterized by $q_h^*$ as $\psi(\cdot,\cdot,\cdot)^\top q_h^*$ at each step $h\in[H]$.
\ORAI{} learns the $Q$-functions of MPPs and the prior of persuasion states simultaneously. 
It operates similarly as that in tabular MPPs and contextual Bayesian persuasion. At the $t$-th episode, given the historical data $\{ (c_h^\tau,s_h^\tau, \omega_h^\tau, a_h^\tau, v_h^\tau) \}_{h\in [H], \tau \in [t-1]}$, we can estimate the unknown vectors $\theta_h^*,q_h^*, \forall h\in [H]$ by solving the following constrained or regularized least-squares problems:
\begin{align*}
\theta_h^t & \gets \argmin_{\Vert \theta_h \Vert \leq L_\theta}  \sum_{\tau \in [t-1]} \big[\omega_h^\tau-f(\phi(c_h^\tau)^\top\theta_h)\big]^2,
\\
q_h^t &  \gets \argmin_{q \in \RR^{d_\psi}} \sum_{\tau \in [t-1]} \big[ v_h^\tau+V_{h+1}^t(s_{h+1}^\tau;C^t)-\psi(s_h^\tau,\omega_h^\tau,a_h^\tau)^\top q \big]^2 + \lambda \Vert q\Vert^2.
\end{align*}
Additionally, $V_{h+1}^t$ is the estimated value function with the observed \cstate{} $C^t$ at the episode $t$, which we describe formally later. This estimator is used to replace the unknown transition $P_h$ and distribution $\nu_h$ in equation \eqref{eq:optimal-Bellman}. Moreover, we can update the estimate of \pstate{} prior $\mu_h^t$ and $Q$-function $Q_h^t$ respectively. Here \ORAI{} adds UCB bonus to $Q_h^t$ to encourage exploration. The formal description is given in Algorithm \ref{alg:mdp}. 
Likewise, the MPP setting inherits the regularity conditions and Assumption \ref{assum:prior} and \ref{assum:f} in the last section. Combining the insights from both the tabular MPPs and contextual Bayesian persuasion, we can show that the \ORAI{} is persuasive and guarantees sublinear regret with high probability for   general MPPs.  

\begin{theorem}\label{thm:mdp_regret}
Under $(p_0,D)$-regularity and Assumption \ref{assum:prior} and \ref{assum:f},  there exists absolute constants $C_1, C_2>0$ such that, if we set $\lambda=\max\{1,\Psi^2\}$, $\beta=C_1(1+\kappa^{-1}\sqrt{K+M+d_\phi\sigma^2\log(HT)})$,and $\rho=C_2 d_\psi H \sqrt{\log(4d_\psi \Psi^2 H^2 T^3)}$, then with probability at least $1-3H^{-1}T^{-1}$, \ORAI{} has regret of order $\tilde{O}\big( d_\phi d_\psi^{3/2}  H^2 \sqrt{T} / (p_0 D) \big)$. 
\end{theorem}

Recall that the novelty of \ORAI{} is that we adopt pessimism and optimism to induce robust equilibrium behavior and encourage exploration simultaneously. Specifically, pessimism tackles the uncertainty in the prior estimation by selecting a signaling policy that is persuasive w.r.t. all the priors in the confidence region, while optimism in $Q$-function estimation encourages exploration. To evaluate the regret of \ORAI{}, we provide a novel regret decomposition, which is tailored to this pessimism and optimism combination. Each term represents different aspects of regret loss incurred by either estimation or randomness.

\section{Proof Sketch and Technical Highlights}\label{sec:proof-sketch}

In this section, we present the proof sketch for Theorem \ref{thm:mdp_regret}. We first decompose the regret into several terms tailored to MPPs and briefly introduce how to bound each term. Then we highlight our technical contribution about regularity when measuring the loss in the sender's utility for choosing a signaling scheme that is persuasive for a subset of priors close to each other.

\subsection{Proof of Theorem \ref{thm:mdp_regret}}

In order to prove the sublinear regret for \ORAI{}, we construct a novel regret decomposition tailored to MPPs.
Our proof starts from decomposing the regret into several terms, each of which indicates the regret loss either from estimation or from the randomness of trajectories. Next, we evaluate each term and then add them together to conclude the upper bound of the regret of \ORAI. For simplicity of presentation, denote $\tilde{V}_h^t(\cdot;C)= \big\langle Q_h^t,\mu_h^*\otimes\pi_h^t\big\rangle_{\Omega\times\cA}(\cdot;C)$ as the expectation of $Q_h^t$ with respect to the ground-truth prior $\mu_h^*$  and signaling scheme $\pi_h^t$ at the $h$-th step. Then we can define the temporal-difference (TD) error as
\begin{equation}\label{eq:defi-delta}
    \delta_h^t(s,\omega,a)=(v_h^t+P_h V_{h+1}^t-Q_h^t)(s,\omega,a;C^t).
\end{equation}
Here $\delta_h^t$ is a function on $\cS \times \Omega \times \cA$ for all $h\in[H]$ and $t\in[T]$. Intuitively, $\{\delta_h^t\}_{h\in[H]}$ quantifies how far the $Q$-functions $\{Q_h^t\}_{h\in[H]}$ are from satisfying the Bellman optimality equation in equation \eqref{eq:optimal-Bellman}. Moreover, define $\zeta_{t,h}^1$ and $\zeta_{t,h}^2$  for the trajectory $\{c_h^t,s_h^t,\omega_h^t,a_h^t\}_{h\in[H]}$ generated by Algorithm \ref{alg:mdp} at the $t$-th episode as follows
\begin{equation}\label{eq:defi-zeta}
    \begin{split}
        \zeta_{t,h}^1&=(\tilde{V}_h^t-V_h^{\pi^t})(s_h^t;C^t)-(Q_h^t-Q_h^{\pi^t})(s_h^t,\omega_h^t,a_h^t;C^t),\\
        \zeta_{t,h}^2&=P_h(V_{h+1}^t-V_{h+1}^*)(s_h^t,\omega_h^t,a_h^t;C^t)-(V_{h+1}^t-V_{h+1}^*)(s_{h+1}^t;C^t).
    \end{split}
\end{equation}
By definition, $\zeta_{t,h}^1$ capture the randomness of  realizing the outcome $\omega_h^t\sim \mu_h^*(\cdot\vert c_h)$ and signaling the action   $ a_h^t  \sim\pi_h^t(s_h^t,\omega_h^t,\cdot)$, while $\zeta_{t,h}^2$ captures the randomness of drawing the next state $s_{h+1}^t$ from $P_h(\cdot\vert s_h^t,\omega_h^t,\cdot)$. With the notations above, we can decompose the regret into six parts to facilitate the establishment of the upper bound of the regret.

\begin{lemma}[Regret Decomposition]\label{lm:mdp-decomp}
With the notations defines in equation \eqref{eq:defi-delta} and \eqref{eq:defi-zeta}, we can write the regret as: 
\begin{equation}\label{eq:regret-decomp}
\begin{split}
    \Reg(T,\mu^*)=& \underbrace{\sum_{t\in[T]}\sum_{h\in[H]} \big\{ \mathbb{E}_{\mu_h^*,\pi_h^*} [ \delta_h^t(s_h,\omega_h,a_h^t) \vert s_1=s_1^t]-\delta_h^t(s_h^t,\omega_h^t,a_h^t)\big\}}_{\displaystyle \mathrm{(i)}}
    +\underbrace{\sum_{t\in[T]}\sum_{h\in[H]}(\zeta_{t,h}^1+\zeta_{t,h}^2)}_{\displaystyle \mathrm{(ii)} }\\
    &~~~~+\underbrace{\sum_{t\in[T]}\sum_{h\in[H]}\mathbb{E}_{\mu_h^*,\pi_h^*}\big[\big\langle Q_h^t,\mu_h^*\otimes\pi_h^*-\mu_h^t\otimes \pi_h^t \big\rangle_{\Omega\times\cA}(s_h;C^t)\vert s_1=s_1^t\big]}_{\displaystyle \mathrm{(iii)}}\\
    &~~~~+\underbrace{\sum_{t\in[T]}\sum_{h\in[H]}\big\langle Q_h^t,(\mu_h^t-\mu_h^*)\otimes\pi_h^t\big\rangle_{\Omega\times\cA}(s_h^t,C^t)}_{\displaystyle \mathrm{(iv)}}.
    \end{split}
\end{equation}
\end{lemma}

In this novel regret decomposition, term $\displaystyle \mathrm{(i)}$ indicates the optimism in \ORAI. Provably, $\delta_h^t$ in term $\displaystyle \mathrm{(i)}$ is always non-positive due to the optimistic $Q$-value estimation,  which could simplify this term. Term $\displaystyle \mathrm{(iii)}$ corresponds to the pessimism in \ORAI{} for inducing a robust equilibria. It evaluates the regret loss incurred by choosing a robustly persuasive signaling policy. Since the signaling policy has to be persuasive to ensure that receivers will always take recommended actions, we cannot simply choose a greedy policy for a fixed prior estimation. Instead, we first apply optimism to construct the optimistic $Q$-value estimation and then apply pessimism to select a signaling policy that is robustly persuasive for all the priors in the confidence region. Therefore, we design the regret decomposition, especially term $\displaystyle \mathrm{(iii)}$ in this form to reflect the optimism and pessimism principle in \ORAI.
Notice that this decomposition does not depend on specific function approximation forms in the algorithm, since not only the estimation of prior and $Q$-function but also the chosen signaling policy has no influence on this formula. Therefore, it generally suits all the algorithms for MPPs.

Unlike the regret decomposition in \cite{yang2020approximation}, Lemma \ref{lm:mdp-decomp} also captures the randomness of realizing the \pstate. Since we have to estimate the prior of the \pstate{} and choose a robustly persuasive policy in MPPs, we add term $\displaystyle \mathrm{(iii)}$ and $\displaystyle \mathrm{(iv)}$ to evaluate the further regret loss.

The rigorous arguments turn out to be technical, and thus we shall defer the proof of most lemmas to the appendix while aiming to present all the key ideas and conclusions in the following. For term $\displaystyle \mathrm{(i)}$ in equation \eqref{eq:regret-decomp}, although we do not observe the trajectories under prior $\mu^*$ and signaling policy $\pi^*$, we can upper bound both  $\delta_h^t$   and  $-\delta_h^t$. The following lemma states this result. 
\begin{lemma}[Optimism]\label{lm:mdp-delta}
There exists an absolute constant $c>0$ such that, for any fixed $\delta\in(0,1)$, if  we set $\lambda=\max\{1,\Psi^2\}$ and $\rho=c d_\psi H \sqrt{\iota}$ in Algorithm \ref{alg:mdp} with $\iota=\log(2d_\psi \Psi^2 T/\delta)$, then with probability at least $1-\delta/2$, we have
\begin{equation*}
    -2\rho \Vert \psi(s,\omega,a)\Vert_{(\Gamma_h^t)^{-1}}\leq \delta_h^t(s,\omega,a) \leq 0.
\end{equation*}
for all $s\in\cS,\omega\in\Omega,a\in\cA, h\in[H]$ and $t\in[T]$.
\end{lemma}

Term $\displaystyle \mathrm{(ii)}$ in equation (\ref{eq:regret-decomp}) can be bounded by Lemma 5.3 from \cite{yang2020approximation} using martingale techniques and the Azuma-Hoeffding inequality \cite{azuma1967weighted}. We state the upper bound for term $\displaystyle \mathrm{(ii)}$ in Lemma \ref{lm:mdp-zeta}.
Moreover, term $\displaystyle \mathrm{(iii)}$ in equation (\ref{eq:regret-decomp}) evaluates the regret loss caused by estimating the prior and choosing a robustly persuasive signaling policy. Here, we apply the robustness gap $\Gap$ defined later to bound this term. 
\begin{lemma}[Bounding Term $\displaystyle \mathrm{(iii)}$]\label{lm:mdp-gap}
On the event of $\{\theta_h^*\in\cB_h^t\}$, under Assumption \ref{assum:prior} and \ref{assum:f}, we have
\begin{align*} 
\sum_{t\in[T]}\sum_{h\in[H]}\mathbb{E}_{\mu_h^*,\pi_h^*} \big[\big\langle Q_h^t,\mu_h^*\otimes\pi_h^*&-\mu_h^t\otimes \pi_h^t \big\rangle_{\Omega\times\cA}(s_h;C^t)\vert s_1=s_1^t\big]\\
&\leq  \bigg(\frac{3H L_\mu K}{p_0D}+\frac{H L_\mu K}{2}\bigg) \beta \sum_{h\in[H]} \sum_{t\in[T]} \Vert \phi(c_h^t)\Vert_{(\Sigma_h^t)^{-1}}. 
\end{align*}

\end{lemma}

It remains to bound term $\displaystyle \mathrm{(iv)}$ in equation \eqref{eq:regret-decomp}. This bound can be derived from Holder inequality and the property of the prior. 
\begin{lemma}[Bounding Term $\displaystyle \mathrm{(iv)}$]\label{lm:mdp-concentration}
On the event of $\{\theta_h^*\in\cB_h^t\}$, under Assumption \ref{assum:prior} and \ref{assum:f}, we have
\begin{equation*}
    \sum_{t\in[T]}\sum_{h\in[H]}\big\langle Q_h^t,(\mu_h^t-\mu_h^*)\otimes\pi_h^t\big\rangle_{\Omega\times\cA}(s_h^t;C^t) \leq H L_\mu K \beta \sum_{h\in[H]} \sum_{t\in[T]} \Vert \phi(c_h^t)\Vert_{(\Sigma_h^t)^{-1}}.
\end{equation*}

\end{lemma}

Now we are ready to prove our main result, Theorem \ref{thm:mdp_regret}. By the decomposition in Lemma \ref{lm:mdp-decomp} and all previous lemmas, let $\beta=C(1+\kappa^{-1}\sqrt{K+M+d_\phi\sigma^2\log(HT)})$, and then we obtain the following upper bound for regret:
\begin{align*}
    \Reg(T,\mu^*) \leq& 4\sqrt{2TH^3\log(2HT)}\\
    &~~~~+ \sum_{t\in[T]}\sum_{h\in[H]}\bigg[ 2\rho \big\Vert \psi(s_h^t,\omega_h^t,a_h^t) \Vert_{(\Gamma_h^t)^{-1}} + \bigg(\frac{3H L_\mu K}{p_0D}+\frac{3H L_\mu K}{2}\bigg)\beta\Vert \phi(c_h^t)\Vert_{(\Sigma_h^t)^{-1}} \bigg],
\end{align*}
With the probability of the given event by Lemma \ref{lm:confidence-region} and appropriately chosen $\delta$ in previous lemmas, the above inequality holds for the probability at least $1-3H^{-1}T^{-1}$.

By Lemma \ref{lm:sum-of-phi}, we have
\begin{align*}
 \Reg(T,\mu^*) &\leq 2\sqrt{2TH^3\log(2HT)} + 2\rho H  \sqrt{2d_\psi T\log\big(1+T\Psi^2/(\lambda d_\psi)\big)}\\
 &\qquad \qquad + \beta H^2 L_\mu K \bigg( \frac{3}{p_0D}+\frac{3}{2}\bigg)  \sqrt{2d_\phi T\log\big(1+T/(d_\phi)\big)}.
\end{align*}
Since $\beta$ is in $\tilde{O}(\sqrt{d_\phi})$ and $\rho$ is in $\tilde{O}(d_\psi H)$, we can conclude that the regret of Algorithm \ref{alg:mdp} is $\tilde{O}\big(d_\phi d_\psi^{3/2}H^2 \sqrt{T}/(p_0 D)\big)$.

\subsection{Inducing Robust Equilibria via Pessimism}
One necessary prerequisite is that the signaling policy given by \ORAI{} has to be persuasive to ensure receivers to take recommended actions. However, the optimal signaling policy that is persuasive for the estimated prior can hardly be also persuasive for the true prior, even if the estimation is quite close to it. To ensure persuasiveness under the prior estimation error, we adopt pessimism principle to select a signaling policy that is robustly persuasive for all the priors in the confidence region. And we shall quantify the extra utility loss suffered by the pessimism principle. 
In this subsection, we start by showing that there exists a robust signaling scheme that suffers only $O(\epsilon)$  utility loss compared to the optimal expected utility of persuasion algorithm designed with precise knowledge of the prior.
Formally, in basic MPP, given any fixed $Q$-function $Q(\cdot,\cdot,\cdot)$, we define the \emph{robustness gap} for some state $s\in\cS$ and any prior $\mu\in \cB\subseteq \Delta(\Omega)$ as  
\begin{equation}\label{eq:defi-gap}
\Gap\big(s,\mu,\cB;Q \big)\triangleq \max_{\pi\in \Ps(\mu,u)}\big\langle Q,\mu\otimes\pi\big\rangle_{\Omega\times\cA}(s)-\max_{\pi\in \Ps(\cB,u)}\big\langle Q,\mu\otimes\pi\big\rangle_{\Omega\times\cA}(s).
\end{equation}
We let $\uB(\mu, \epsilon) = \{  \mu' \in \Delta(\Omega) : \norm{\mu - \mu'}_1 \leq \epsilon \}$ be the  $\ell_1$-norm ball centered the prior distribution $\mu$ with radius  $\epsilon$.

\begin{lemma}[Pessimism]
\label{lm:robustness-gap} 
Under $(p_0,D)$-regularity, for all $\epsilon > 0$, given a $Q$-function $Q$, for any state $s\in\cS$, we have
\begin{equation*}
     \Gap \big (s,\mu, \uB(\mu, \epsilon)  ; Q \big)  \leq \frac{H\epsilon}{p_0 D}.
\end{equation*}
\end{lemma}
The proof is given in Appendix \ref{sec:robustness-gap}. This result extends Proposition 1 in \cite{10.1145/3465456.3467593}. Notice that the upper bound of $\Gap (\cdot; \cdot )$ does not depend on the value of $Q$, which is important for our analysis. Once given a signaling algorithm, at each episode $t\in[T]$ and each step $h\in[H]$, we are able to obtain an estimation of $Q$-function with an explicit form. It is equivalent to the ``known'' $Q$-function mentioned in equation \eqref{eq:defi-gap}. Using $\Gap (\cdot; \cdot )$, we can estimate the expected sender's utility loss for choosing a signaling mechanism that is persuasive for all priors in a subset. Moreover, if we consider the dependence on context for priors and add the linear assumption of priors to the proceeding lemma, we can bound $\Gap (\cdot; \cdot )$ by the difference of linearity parameter $\theta$. 
\begin{corollary}\label{cor:gap-linear}
Under $(p_0,D)$-regularity and Assumption \ref{assum:prior} and \ref{assum:f}, given a $Q$-function $Q$ and \cstate{} $c$, for any state $s\in\cS$, prior $\mu_{\theta}(\cdot | c)$ and confidence region $\cB = \{ \mu_{\theta'}( \cdot | c) :\theta' \in \uB_{\Sigma}(\theta, \epsilon) \}$, we have
$ \Gap(s,\mu_{\theta}(\cdot|c), \cB;Q) \leq  H L_{\mu}K \Vert\phi(c)\Vert_{\Sigma^{-1}}  \epsilon / (p_0 D).$

\end{corollary}

In MPPs, we have to estimate the prior of the outcome since we cannot observe the ground-truth prior. However, the estimation may not satisfy the regularity conditions, which conflicts with the requirements for the prior when proving Lemma \ref{lm:robustness-gap}. To address this problem, we give another upper bound of the robustness gap for the prior estimation in Lemma \ref{lm:gap-for-est}. In addition, to handle the regret loss incurred by estimating the prior, we compute the difference in $Q$-functions when choosing respectively persuasive scheme for different priors in Lemma \ref{lm:opt-gap}.   

We now prove  that the above pessimism design guarantees persuasiveness w.r.t. the true prior with high probability. And it suffices to show that the estimation $\theta_h^t$ is close enough to the real parameter $\theta_h^*$ such that the confidence region $\cB_h^t$ centered at $\theta_h^t$ given in Algorithm \ref{alg:mdp} contains $\theta_h^*$. If so, the signaling scheme chosen to be persuasive for the whole set $\mu_{\cB_h^t}$ is also persuasive for $\mu_h^*$, where $\mu_\cB \coloneqq \{ \mu_{\theta'}: \theta' \in \cB\}$ denotes the set of priors that are determined by the parameters $\theta'\in \cB$.

\begin{lemma}\label{thm:mdp-persuasive}
    There exists a constant $C>0$, such that for $\beta=C(1+\kappa^{-1}\sqrt{K+M+d_\phi\sigma^2\log(HT)})$, \ORAI{} Algorithm is persuasive with probability at least $1-H^{-1}T^{-1}$, i.e.,
	\begin{equation*}
	\PP_{\theta^*}\bigg(\bigcup_{h\in[H]}\big\{ \theta_h^* \notin \cap_{t \in [T]}\cB_h^t\big\} \bigg) \leq H^{-1}T^{-1}.
	\end{equation*}
\end{lemma}
\begin{proof}
We first analyze the probability for being non-persuasive. For any $\Vert\theta_h^*\Vert\leq L_\theta$, using the union bound, we have
    \begin{align*}
        P_{\theta^*}\bigg(\bigcup_{t\in[T],h\in[H]}\big\{\theta_h^* \notin \cap_{t \in [T]}\cB_h^t\big\}\bigg) &\leq	\sum_{t\in[T]}\sum_{h\in[H]}P_{\theta_h^*}\big(\theta_h^* \notin \cap_{t \in [T]}\cB_h^t\big)\\
        &\leq \sum_{t\in[T]} \sum_{h\in[H]}  P_{\theta_h^*} \big (\Vert \theta_h^t-\theta_h^*\Vert_{\Sigma_h^t}>\beta \big ).
    \end{align*}
  
 The following lemma gives the belief of confidence region for the linear parameter $\theta_h^*$. The proof can be directly derived from Lemma 6 in \citet{wang2019optimism}.   
\begin{lemma}[Belief of Confidence Region]\label{lm:confidence-region}
For any $t\in[T]$ and $h\in[H]$, there exists a constant $C>0$, such that for $\beta=C(1+\kappa^{-1}\sqrt{K+M+d_\phi\sigma^2\log(1/\delta)})$, given $\delta\in(0,1)$, with probability at least $1-\delta$, we have
$\Vert \theta_h^t-\theta_h^*\Vert_{\Sigma_h^t}\leq  \beta. $
  
\end{lemma}

	By Lemma \ref{lm:confidence-region}, taking  $\delta=H^{-2}T^{-2}$, then we have $\PP_{\theta^*}(\Vert \theta_h^t-\theta_h^*\Vert_{\Sigma_h^t}>\beta) \leq H^{-2}T^{-2}$. Summing up the failure  probabilities over $t \in [T]$, we have  $	\PP_{\theta^*}(\theta^* \notin \cap_{t \in [T]}\cB^t) \leq H^{-1}T^{-1}$.
\end{proof}

\section{Conclusion}
We have presented a novel model, the MPP, which captures the misaligned incentives of uninformed decision makers and the long-term objective of an information possessor for the first time. We then provide a  reinforcement learning algorithm, \ORAI{}, that is provably efficient in terms of both computational complexity and sample complexity, under mild assumptions. We remark that while we showcase this algorithm in particular problem instances with linear approximation or GLMs, the framework of \ORAI{} does not rely on the function approximation form, as long as we can quantify the uncertainty of the prior estimation and $Q$-function (or transition model). In addition, we expect this optimism-pessimism design principle and its corresponding proof techniques to be generally useful for some other strategic learning problems with misaligned incentives involved.

Besides extending our techniques to other design problems, we point out that several other open problems arises from our work. First, while it is natural that the sender have knowledge of receiver's utility functions in many cases (see Footnote \ref{fn:receiver-payoff}), we hope to also study the problem even without initially knowing receiver's utility.
Similar problem has been studied in Stackelberg games~\cite{peng2019learning, conitzer2006computing} yet without measuring the performance in terms of the cumulative utility of sender (leader).
Second, another interesting direction is to study the setting of Markov Bayesian persuasion with one sender and one receiver, both aiming at maximizing their own long-term cumulative utilities when the environment involves Markovian transitions.

\bibliographystyle{ims} 
\bibliography{content/refer}

\newpage

\appendix
\section{Omitted Proofs and Descriptions}
\subsection{Formal Description of the \ORAI}

The formal description of the \ORAI{} for MPPs is stated as follows:
\begin{algorithm}[H]
\caption{The \ORAI{} for MPPs}
\label{alg:mdp}
\begin{algorithmic}[1] %
\STATE \textbf{Input:} Number of Episodes $T$, Number of Step $H$ 
\STATE \textbf{Parameters:}  $ \beta > 0$, $\rho > 0$, $\lambda\in \RR^+$. %
\STATE \textbf{Output:} $a_h^t \in \cA$ for each $h\in[H], t \in [T]$
\FOR{episode $t = 1\dots T$} 
\STATE Receive the initial state $s_1^t$ and \cstate{} $C^t=(c_1^t,\ldots,c_H^t)$.
\FOR{step $h=H,\ldots,1$}
\STATE Compute the constrained least square problem \begin{equation*}
\theta_h^t \gets \argmin_{\Vert \theta_h \Vert \leq L_\theta}  \sum_{\tau \in [t-1]} \big[\omega_h^\tau-f(\phi(c_h^\tau)^\top\theta_h)\big]^2.
\end{equation*}
\STATE Calculate $\Sigma_h^t=\Phi^2 I_{d_\phi}+\sum_{\tau\in[t-1]}\phi(c_h^\tau)\phi(c_h^\tau)^\top$. Update $\cB_h^{t} \gets \uB_{\Sigma_h^t}(\theta_h^t,\beta)$.
\STATE Set $\mu_h^t(\cdot\vert c)$ to the distribution of $f\big(\phi(c)^\top \theta_h^t\big)+z_h$.
\STATE Calculate \begin{align*}
\Gamma_h^t  &=\lambda I_{d_\psi}+\sum_{\tau \in [t-1]}\psi(s_h^\tau,\omega_h^\tau,a_h^\tau)\psi(s_h^\tau,\omega_h^\tau,a_h^\tau)^\top, \\
\iota_h^t &=\sum_{\tau \in [t-1]}\psi(s_h^\tau,\omega_h^\tau,a_h^\tau)[v_h^\tau+V_{h+1}^t(s_{h+1}^\tau;C^t)]
\end{align*}
\STATE Update $q_h^t\gets(\Gamma_h^t)^{-1}\iota_h^t$. %
\STATE Set $\begin{cases}
 Q_h^t(\cdot,\cdot,\cdot;C^t) \gets \min\{\psi(\cdot,\cdot,\cdot)^\top q_h^t+\rho\Vert \psi(\cdot,\cdot,\cdot)\Vert_{(\Gamma_h^t)^{-1}},H\},&  \\ V_h^t(\cdot;C^t)\gets \max_{\pi_h\in\Ps(\mu_{\cB_h^t},u_h)}\big\langle Q_h^t,\mu_h^t\otimes\pi_h\big\rangle_{\Omega\times\cA}(\cdot;C^t).&
\end{cases} $
\ENDFOR
\FOR{step $h=1,\ldots,H$}
\STATE Choose $\pi_h^t\in \arg\max_{\pi_h\in\Ps(\mu_{\cB_h^t},u_h)} \big\langle Q_h^t,\mu_h^t\otimes\pi_h\big\rangle_{\Omega\times\cA}(s_h^t;C^t)$.		\ENDFOR
\STATE Execute $\pi^t$ to sample a trajectory $\{ (s_h^{t}, \omega_h^{t}, a_h^{t}, v_h^{t}) \}_{h\in [H]}$.
\ENDFOR
\end{algorithmic}
\end{algorithm}

\subsection{Proof of Lemma \ref{lm:robustness-gap} }\label{sec:robustness-gap}
\begin{proof}

We prove with an explicit construction of a signaling scheme that is robustly persuasive for any prior in $\uB(\mu, \epsilon)$ and achieve the expected utility at least $\max_{\pi\in \Ps(\mu,u)}\big\langle Q,\mu\otimes\pi\big\rangle_{\Omega\times\cA}(s) - H\epsilon/(p_0 D) $. To simplify the notation, we omit the $s$ in $u$, $Q$ and $\cW$.

Let $\pi^* = \arg\max_{\pi \in \Ps(\mu,u)} \big\langle Q,\mu\otimes\pi\big\rangle_{\Omega\times\cA}$ be a direct scheme without loss of generality~\cite{kamenica2011bayesian}. 
For each $a \in \cA$, let $\mu_a(\cdot) \coloneqq \mu(\cdot) \odot \pi^*(a|\cdot)  $ denote the posterior of outcome (i.e., kernel \footnote{In this proof, we will directly work with the posterior without normalization (kernel) to simplify our notations and derivations, because $ \int_{\omega \in \Omega} \mu_a(\omega) \left[ u(\omega, a ) - u(\omega, a' ) \right] \geq 0 \iff \int_{\omega \in \Omega} \frac{\mu_a(\omega)}{\int_{\omega \in \Omega} \mu_a(\omega) } \left[ u(\omega, a ) - u(\omega, a' ) \right] \geq 0$. We use $\odot$ to denote the Hadamard product.}) that action $a$ is recommended by $\pi$, so the prior can be composed as $\mu(\cdot) = \sum_{a\in \cA} \mu_a(\cdot)$. Since $\pi$ is persuasive, we know $ \int_{\omega \in \Omega} \mu_a(\omega) \left[ u(\omega, a ) - u(\omega, a' ) \right] \geq 0, \forall a' \in \cA.$

Let $\pi^0$ be the fully revealing signaling scheme that always recommends (signals) the action that maximizes the receivers' utility at the realized outcome. For each $a \in \cA$, let $\eta_a(\cdot) \coloneqq \mu(\cdot) \odot \pi^0(a|\cdot)$ denote the posterior of outcome that action $a$ is recommended by $\pi^0$, so the prior can be composed as $\mu(\cdot) = \sum_{a\in \cA} \eta_a(\cdot)$. By regularity condition, we have $$ \int_{\omega \in \Omega} \eta_a(\omega) \left[ u(\omega, a ) - u(\omega, a' ) \right] \geq \int_{\omega \in \cW_{a}(D)} \eta_a(\omega) \left[ u(\omega, a ) - u(\omega, a' ) \right] \geq p_0 D,\quad \forall a'\in \cA.$$

We now show that the signaling scheme $\pi' = (1-\delta) \pi^* + \delta \pi^0$ is persuasive for any prior $\tilde{\mu} \in \uB(\mu, \epsilon)$ with $\delta = \frac{\epsilon}{p_0 D}$. One simple way to interpret this ``compound'' signaling scheme is to follow $ \pi^*$ with probability $(1-\delta)$ and follow $ \pi^0$ with probability $\delta$. Hence, given a recommended action $a$, the receiver would compute the posterior as $\mu'_a = (1-\delta)\mu_a(\omega) + \delta \eta_a(\omega)$.
Let $\mu'_a, \tilde{\mu}_a$ be the outcome posterior of $\pi'$ recommending action $a$ under the true prior $\mu$ (resp. the perturbed prior $\tilde{\mu}$). So $\mu'_a(\cdot) = \mu(\cdot) \odot \pi'(a|\cdot) $ and $\tilde{\mu}_a(\cdot) = \tilde{\mu}(\cdot) \odot \pi'(a|\cdot) $.  By definition of persuasiveness, we need to show that for any recommended action (signal from $\pi'$) $a\in \cA$, the action $a$ maximizes the receiver's utility under  $\mu'_a $. This follows from the decomposition below,
\begin{align*}
    & \int_{\omega \in \Omega} \tilde{\mu}_a \cdot \left[
     u(\omega, a ) - u(\omega, a' ) \right] \\
 \geq & \int_{\omega \in \Omega} \mu'_a \cdot \left[ u(\omega, a ) - u(\omega, a' ) \right]  - \norm{ \tilde{\mu}_a - \mu'_a }_1  \\
 \geq & \int_{\omega \in \Omega} \left[ (1-\delta)\mu_a(\omega) + \delta \eta_a(\omega) \right] \cdot \left[ u(\omega, a ) - u(\omega, a' ) \right]  - \norm{ \tilde{\mu}_a - \mu_a }_1  \\
  = & \int_{\omega \in \Omega} (1-\delta)\mu_a(\omega) \left[ u(\omega, a ) - u(\omega, a' ) \right] + \int_{\omega \in \Omega} \delta \eta_a(\omega) \left[ u(\omega, a ) - u(\omega, a' ) \right]   - \norm{ \tilde{\mu}_a - \mu_a }_1  \\
\geq    & \ \delta p_0 D - \norm{ \tilde{\mu}_a - \mu_a }_1   \\
=    & \ \epsilon - \norm{ \tilde{\mu}_a - \mu_a }_1 \geq  0.  
\end{align*}
The first inequality is by the fact that $u(\omega, a)\in [0, 1]$ for any $\omega, a$ and thus $\sum_{a} ( \tilde{\mu}_a - \mu'_a ) \cdot \left[ u(\omega, a ) - u(\omega, a' ) \right] \leq \norm{ \tilde{\mu}_a - \mu'_a }_1  $. The second inequality is from $\mu'_a = (1-\delta)\mu_a(\omega) + \delta \eta_a(\omega)$. The third inequality is by construction of $\mu_a$ and $\eta_a$ induced by signaling scheme $\pi$ and $\pi^0$. The last inequality is by the fact that $\norm{ \tilde{\mu}_a - \mu'_a }_1 = \norm{ (\tilde{\mu} - \mu')\odot \pi'(a|\cdot) }_1 \leq \norm{ \tilde{\mu} - \mu' }_1 = \epsilon$, since $\norm{\pi'(a|\cdot) }_{\infty} \leq 1$

It remains to show the expected utility under signaling scheme $\pi'$ is at least $\big\langle Q,\mu\otimes\pi^*\big\rangle_{\Omega\times\cA}- H\epsilon/(p_0 D) $. This is due to the following inequalities,
\begin{align*}
    \big\langle Q,\mu\otimes\pi'\big\rangle_{\Omega\times\cA} - \big\langle Q,\mu\otimes\pi^*\big\rangle_{\Omega\times\cA}
    & = \int_{\omega \in \Omega, a\in \cA} \mu(\omega) \left[ \pi'(a|\omega) - \pi^*(a|\omega) \right] Q(\omega, a)  \\
     & = \int_{\omega \in \Omega, a\in \cA} \mu(\omega) \left[ \delta \pi^0(a|\omega) -\delta\pi^*(a|\omega)  \right] Q(\omega, a) \\
     & \geq -\delta \int_{\omega \in \Omega, a\in \cA} \mu(\omega) \pi(a|\omega) Q(\omega, a)  \\
     & \geq -H\delta = - \frac{H\epsilon}{p_0 D}.
\end{align*}
The first and second equalities use the definition and linearity. The third and last inequalities use the fact that $\EE[Q(\omega,a)]\in [0, H]$ and remove the positive term.
\end{proof}

\subsection{Properties for the Robustness Gap}\label{sec:prop-gap}

We present the robustness gap $\Gap$ for the ground-truth prior in Lemma \ref{lm:robustness-gap}. For the estimation of prior $\mu_h^t$ given in Algorithm \ref{alg:mdp} which may not satisfy the regularity condition, we also have corresponding robustness gap.
\begin{lemma}\label{lm:gap-for-est}
For any $h\in[H], t\in[T]$ and $s\in\cS$, on the event of $\{\theta_h^*\in \cB_h^t\}$, we have
\begin{equation*}
    \Gap(s,\mu_h^t,\uB(\mu_h^t,\epsilon_h^t);Q_h^t) \leq \frac{2H\epsilon}{p_0 D}.
\end{equation*}
\end{lemma}
\begin{proof}
For any fixed action $a \in\cA$, on the given event, we have 
\begin{align*}
    \PP_{\omega\sim\mu_h^t(\cdot)} [\omega \in \cW_{s,a}(D)]&=\int_{\omega\in\Omega} \mu_h^t(\omega) \II (\omega \in \cW_{s,a}(D)) \ud \omega \\
    &= \int_{\omega\in\Omega} \mu_h^*(\omega ) \II (\omega \in \cW_{s,a}(D)) \ud \omega + \int_{\omega\in\Omega} [\mu_h^t(\omega )-\mu_h^*(\omega )] \II (\omega \in \cW_{s,a}(D)) \ud \omega \\
    &\geq \int_{\omega\in\Omega} \mu_h^*(\omega ) \II (\omega \in \cW_{s,a}(D)) \ud \omega + \Vert \mu_h^t-\mu_h^*\Vert_1\\
    &\geq p_0- \epsilon_h^t,
\end{align*}
where $\II$ is the indicating function. The last inequality uses the regularity condition for the real prior $\mu_h^*$. For $\epsilon_h^t \leq p_0/2$, we have $\PP_{\omega\sim\mu_h^t(\cdot )} [\omega \in \cW_{s,a}]\leq p_0/2$. Then by Lemma \ref{lm:robustness-gap}, we can arrive at
\begin{equation*}
    \Gap(s,\mu_h^t,\uB(\mu_h^t,\epsilon_h^t);Q_h^t) \leq \frac{2H\epsilon_h^t}{p_0 D}.
\end{equation*}
For $\epsilon_h^t > p_0/2$, the bound holds trivially since $2H\epsilon_h^t/(p_0 D)>H$.
\end{proof}

The robustness gap $\Gap$ defined in equation \eqref{eq:defi-gap} measures the loss in value functions for being robustly persuasive for a subset of priors. In the following lemma, we show that we can also use $\Gap$ to bound the difference in expected optimal $Q$-functions between different priors.
\begin{lemma}\label{lm:opt-gap}
Denote $\cB_{1,2} \coloneqq \uB\big(\mu_1,\Vert \mu_1-\mu_2\Vert_1\big)$ for any fixed state $s\in\cS$ and $\mu_1,\mu_2\in\Delta(\Omega)$.  Then given a known $Q$-function $Q(\cdot,\cdot,\cdot)$, we have 
\begin{equation*}
    \max_{\pi_1\in \Ps(\mu_1,u)} \big\langle Q,\mu_1\otimes\pi_1\big\rangle_{\Omega\times\cA}(s)-\max_{\pi_2\in \Ps(\mu_2,u)} \big\langle Q,\mu_2\otimes\pi_2\big\rangle_{\Omega\times\cA}(s) \leq \Gap(s,\mu_1,\cB_{1,2};Q)+\frac{H}{2} \Vert\mu_1-\mu_2\Vert_1.
\end{equation*}
\end{lemma}
\begin{proof}
Fix $\mu_1,\mu_2\in\Delta(\Omega)$, we respectively choose the optimal signaling scheme \begin{equation*}
    \pi_i=\argmax_{\pi_i\in \Ps(\mu_i,u)} \big\langle Q,\mu_i\otimes\pi_i\big\rangle_{\Omega\times\cA}(s), ~i=1,2.
\end{equation*}
Then among all the signaling schemes persuasive for all $\cB_{1,2}$, let $\pi_3$ maximize $\big\langle Q,\mu_1\otimes\pi\big\rangle_{\Omega\times\cA}(s)$. Since $\pi_3$ is persuasive for $\mu_2$, we know $\big\langle Q,\mu_2\otimes\pi_2\big\rangle_{\Omega\times\cA}(s) \geq  \big\langle Q,\mu_2\otimes\pi_3\big\rangle_{\Omega\times\cA}(s)$ by definition. Therefore, we have
\begin{align*}
    \big\langle Q,\mu_1\otimes\pi_1-\mu_2\otimes\pi_2\big\rangle_{\Omega\times\cA}(s)  &\leq \big\langle Q,\mu_1\otimes\pi_1-\mu_2\otimes\pi_3 \big\rangle_{\Omega\times\cA}(s) \\
    &\leq \big\langle Q,\mu_1\otimes\pi_1 - \mu_1\otimes\pi_3 \big\rangle_{\Omega\times\cA}(s)+\big\langle Q,\mu_1\otimes\pi_3 - \mu_2\otimes\pi_3 \big\rangle_{\Omega\times\cA}(s)\\
    &= \Gap(s,\mu_1,\cB_{1,2};Q)+\frac{H}{2} \Vert\mu_1-\mu_2\Vert_1.
\end{align*}
The last equality uses the definition of $\Gap$ and Lemma \ref{lm:prior-diff}. 
\end{proof}

\begin{lemma}\label{lm:prior-diff}
Given a $Q$-function $Q(\cdot,\cdot,\cdot)\in [0,H]$, for any fixed state $s\in\cS$, $\mu_1,\mu_2\in\Delta(\Omega)$ and any signaling scheme $\pi$, we have
\begin{equation*}
    \big\vert \big\langle Q,\mu_1\otimes\pi\big\rangle_{\Omega\times\cA}(s)-  \big\langle Q,\mu_2\otimes\pi\big\rangle_{\Omega\times\cA}(s) \big\vert \leq \frac{H}{2} \Vert\mu_1-\mu_2\Vert_1.
\end{equation*}
\end{lemma}
\begin{proof}
Fix $\mu_1(\cdot ),\mu_2(\cdot )\in\Delta(\Omega)$. For any $x\in \RR$, we have
\begin{align*}
\big\vert \big\langle Q,\mu_1\otimes\pi-\mu_2\otimes\pi\big\rangle_{\Omega\times\cA}(s)&= \bigg\vert \int_{\omega\in\Omega}[\mu_1(\omega )-\mu_2(\omega )]\bigg[\int_{a\in\cA}\pi(a|s,\omega)Q(s,\omega,a)\ud a-x \bigg]\ud \omega \bigg\vert\\ 
& \leq \Vert \mu_1-\mu_2\Vert_1 \cdot  \sup_{\omega\in\Omega} \bigg\vert\int_{a\in\cA}\pi(a|s,\omega)Q(s,\omega,a)\ud a-x \bigg\vert,
\end{align*}
where the last inequality is derived from Holder's inequality. With $Q$-function taking values in $[0,H]$, we can set $x=H/2$ and achieve the optimality.

\end{proof}

\subsection{Proof of Lemma \ref{lm:mdp-decomp}}\label{sec:proof-decomp}
\begin{proof}
    Before presenting the proof, we first define two operators $\mathbb{J}_h^*$ and $\mathbb{J}_h^t$:
    \begin{equation}\label{eq:defi-J}
        (\mathbb{J}_h^* f)(s;C) =\langle f,\mu_h^*\otimes\pi_h^*\rangle_{\Omega\times\cA}(s;C),~~~~(\mathbb{J}_h^t f)(s;C) = \langle f,\mu_h^t\otimes\pi_h^t\rangle_{\Omega\times\cA}(s;C),
    \end{equation}
    for any $h\in[H],t \in[T]$ and any function $f(\cdot,\cdot,\cdot;C):\cS\times\Omega\times\cA\to\RR$ under the \cstate{} $C$. Moreover, for any $h\in[H],t \in[T]$ and any state $s\in\cS$, we define
    \begin{equation}\label{eq:defi-xi}
        \xi_h^t(s;C)= (\mathbb{J}_h^* Q_h^t)(s;C)- (\mathbb{J}_h^t Q_h^t)(s;C)=
        \langle Q_h^t,\mu_h^*\otimes\pi_h^*-\mu_h^t\otimes\pi_h^t\rangle_{\Omega\times\cA}(s;C).
    \end{equation}
    After introducing these notations, we decompose the instantaneous regret at the $t$-th episode into two terms,
    \begin{equation}\label{eq:decomp}
        V_1^*(s_1^t;C^t)-V_1^{\pi^t}(s_1^t;C^t)=\underbrace{V_1^*(s_1^t;C^t)-V_1^t(s_1^t;C^t)}_{\bold{p}_1}+\underbrace{V_1^t(s_1^t;C^t)-V_1^{\pi^t}(s_1^t;C^t)}_{\bold{p}_2}.
    \end{equation}
    Then we consider these two terms separately. By the definition of value functions in (\ref{eq:policy-Bellman}) and the operator $\mathbb{J}_h^*$ in (\ref{eq:defi-J}), we have $V_h^*=\mathbb{J}_h^*Q_h^*$. By the construction of Algorithm \ref{alg:mdp}, we have $V_h^t=\mathbb{J}_h^t Q_h^t$ similarly. Thus, for the first term $\bold{p}_1$ defined in equation \eqref{eq:decomp}, using $\xi_h^t$ defined in (\ref{eq:defi-xi}), for any $h\in[H],t\in[T]$, we have
    \begin{equation*}
    \begin{split}
        V_h^*-V_h^t&=\mathbb{J}_h^*Q_h^*-\mathbb{J}_h^t Q_h^t=(\mathbb{J}_h^*Q_h^*-\mathbb{J}_h^*Q_h^t)+(\mathbb{J}_h^*Q_h^t-\mathbb{J}_h^t Q_h^t)\\
        &=\mathbb{J}_h^*(Q_h^*-Q_h^t)+\xi_h^t.
    \end{split}
    \end{equation*}
    Next, by the definition of the temporal-difference error $\delta_h^t$ in (\ref{eq:defi-delta}) and the Bellman optimality equation in equation \eqref{eq:optimal-Bellman}, we have
    \begin{equation*}
        Q_h^*-Q_h^t=(v_h+P_h V_{h+1}^*)-(v_h+P_h V_{h+1}^t-\delta_h^t)=P_h ( V_{h+1}^*-V_{h+1}^t)+\delta_h^t.
    \end{equation*}
    Hence we get 
    \begin{equation*}
        V_h^*-V_h^t=\mathbb{J}_h^*P_h ( V_{h+1}^*-V_{h+1}^t)++\mathbb{J}_h^*\delta_h^t+\xi_h^t.
    \end{equation*}
    Then, by recursively applying the above formula, we have
    \begin{align*}
        V_1^*-V_1^t= \bigg(\prod_{h\in[H]}\mathbb{J}_h^*P_h\bigg)(V_{H+1}^*-V_{H+1}^t)  +\sum_{h\in[H]}\bigg(\prod_{i\in[h]}\mathbb{J}_i^*P_i \bigg) \mathbb{J}_h^*\delta_h^t+ \sum_{h\in[H]}\bigg(\prod_{i\in[h]}\mathbb{J}_i^*P_i \bigg) \xi_h^t.
    \end{align*}
    By the definition of $\xi_h^t$ in equation \eqref{eq:defi-xi} and $\zeta_{t,h}^3$ in equation \eqref{eq:defi-zeta}, we get 
    \begin{align*}
        \sum_{h\in[H]}\bigg(\prod_{i\in[h]}\mathbb{J}_i^*P_i \bigg) \xi_h^t(s_h;C^t) = \sum_{h\in[H]}\mathbb{E}_{\mu^*,\pi^*}\big\{\big[ \langle Q_h^t,\mu_h^*\otimes\pi_h^*-\mu_h^t\otimes\pi_h^t\rangle_{\Omega\times\cA}(s_h;C^t)|s_1  =s_1^t \big]\big\}.
    \end{align*}
    Notice that $V_{H+1}^*=V_{H+1}^t=0$. Therefore, for any episode $t\in[T]$, we have
    \begin{equation*}
    \begin{split}
        V_1^*(s_1^t;C^t)-V_1^t(s_1^t;C^t)
        =&\sum_{h\in[H]}\mathbb{E}_{\mu^*,\pi^*}\big\{\big[ \langle Q_h^t,\mu_h^*\otimes\pi_h^*-\mu_h^t\otimes\pi_h^t\rangle_{\Omega\times\cA}(s_h;C^t)  |s_1 =s_1^t \big]\big\}\\
        &~~~~+\sum_{h\in[H]}\mathbb{E}_{\mu^*,\pi^*} \left[ \delta_h^t(s_h,\omega_h,a_h)  |s_1=s_1^t\right].
    \end{split}
    \end{equation*}
    Now we come to bound the second term $\bold{p}_2$ in equation \eqref{eq:decomp}. By the definition of the temporal-difference error $\delta_h^t$ in (\ref{eq:defi-delta}), for any $h\in[H],t\in[T]$, we note that
    \begin{align*}
        \delta_h^t(s_h^t,\omega_h^t,a_h^t)&=(v_h^t+P_h V_{h+1}^t-Q_h^t)(s_h^t,\omega_h^t,a_h^t;C^t)\\
        &=(v_h^t+P_h V_{h+1}^t-Q_h^{\pi^t})(s_h^t,\omega_h^t,a_h^t;C^t)+(Q_h^{\pi^t}-Q_h^t)(s_h^t,\omega_h^t,a_h^t;C^t)\\
        &=(P_h V_{h+1}^t-P_h V_{h+1}^{\pi^t})(s_h^t,\omega_h^t,a_h^t)+(Q_h^{\pi^t}-Q_h^t)(s_h^t,\omega_h^t,a_h^t).
    \end{align*}
    where the last equality follows the Bellman equation (\ref{eq:policy-Bellman}). Furthermore, using $\zeta_{t,h}^1$ and $\zeta_{t,h}^2$ defined in (\ref{eq:defi-zeta}), we have
    \begin{align*}
        V_h^t(s_h^t;&C^t)-V_h^{\pi^t}(s_h^t;C^t)\\
        =&(V_h^t-V_h^{\pi^t})(s_h^t;C^t)-\delta_h^t(s_h^t,\omega_h^t,a_h^t)+(Q_h^{\pi^t}-Q_h^t)(s_h^t,\omega_h^t,a_h^t;C^t)\\
        &\quad+(P_h V_{h+1}^t-P_h V_{h+1}^{\pi^t})(s_h^t,\omega_h^t,a_h^t;C^t)
        \\
        =&(V_h^t-\tilde{V}_h^t)(s_h^t;C^t)-\delta_h^t(s_h^t,\omega_h^t,a_h^t)+(\tilde{V}_h^t-V_h^{\pi^t})(s_h^t;C^t)+(Q_h^{\pi^t}-Q_h^t)(s_h^t,\omega_h^t,a_h^t;C^t)\\
        &\quad+\big(P_h (V_{h+1}^t- V_{h+1}^{\pi^t})\big)(s_h^t,\omega_h^t,a_h^t;C^t)-(V_{h+1}^t- V_{h+1}^{\pi^t})(s_{h+1}^t;C^t)+(V_{h+1}^t- V_{h+1}^{\pi^t})(s_{h+1}^t;C^t)
        \\
        =&\big[ V_{h+1}^t(s_{h+1}^t;C^t) - V_{h+1}^{\pi^t}(s_{h+1}^t;C^t) \big]+ \big[ V_h^t(s_h^t;C^t) - \tilde{V}_h^t(s_h^t;C^t) \big] -\delta_h^t(s_h^t,\omega_h^t,a_h^t)+ \zeta_{t,h}^1+\zeta_{t,h}^2.
    \end{align*}
    Applying the above equation recursively, we get that
    \begin{align*}
            V_1^t(s_1^t;C^t)-V_1^{\pi^t}(s_1^t;C^t)=&V_{H+1}^t(s_H^t;C^t)-V_{H+1}^{\pi^t}(s_H^t;C^t)+\sum_{h\in[H]}\big[ V_h^t(s_h^t;C^t) - \tilde{V}_h^t(s_h^t;C^t) \big]\\
            &~~~~-\sum_{h\in[H]}\delta_h^t(s_h^t,\omega_h^t,a_h^t)+\sum_{h\in[H]}(\zeta_{t,h}^1+\zeta_{t,h}^2).
    \end{align*}
    Again by Bellman equation (\ref{eq:policy-Bellman}), we have, 
    \begin{align*}
     V_h^t(s_h^t;C^t) - \tilde{V}_h^t(s_h^t;C^t)&=\big\langle Q_h^t,(\mu_h^t-\mu_h^*)\otimes\pi_h^t \big\rangle_{\Omega\times\cA}(s_h^t;C^t).
    \end{align*}
    Then we use $V_{H+1}^t=V_{H+1}^{\pi^t}=0$ to simplify the decomposition to the following form:
    \begin{align*}
        V_1^t(s_1^t;C^t)-V_1^{\pi^t}(s_1^t;C^t)
        =&\sum_{h\in[H]} \big\langle Q_h^t,(\mu_h^t-\mu_h^*)\otimes\pi_h^t \big\rangle_{\Omega\times\cA}(s_h^t;C^t)\\
        &\quad-\sum_{h\in[H]}\delta_h^t(s_h^t,\omega_h^t,a_h^t)+\sum_{h\in[H]}(\zeta_{t,h}^1+\zeta_{t,h}^2).
    \end{align*}
   Therefore, combining $\bold{p}_1$ and $\bold{p}_2$, we can conclude the proof of this lemma.
    \begin{align*}
         \Reg(T,\mu^*)=&\sum_{t\in[T]}\left[ V_1^*(s_1^t;C^t)-V_1^{\pi^t}(s_1^t;C^t)\right]\\
         =&\sum_{t\in[T]}\sum_{h\in[H]} \big\{ \mathbb{E}_{\mu_h^*,\pi_h^*} [ \delta_h^t(s_h,\omega_h,a_h)  |s_1=s_1^t]-\delta_h^t(s_h^t,\omega_h^t,a_h^t)\big\}
        +\sum_{t\in[T]}\sum_{h\in[H]}(\zeta_{t,h}^1,\zeta_{t,h}^2)\\
        &~~~~+\sum_{t\in[T]}\sum_{h\in[H]}\mathbb{E}_{\mu_h^*,\pi_h^*}\big[\big\langle Q_h^t,\mu_h^*\otimes\pi_h^*-\mu_h^t\otimes \pi_h^t \big\rangle_{\Omega\times\cA}(s_h;C^t) |s_1 =s_1^t\big]\\
        &~~~~+\sum_{t\in[T]}\sum_{h\in[H]}\big\langle Q_h^t,(\mu_h^t-\mu_h^*)\otimes\pi_h^t \big\rangle_{\Omega\times\cA}(s_h^t;C^t).
    \end{align*}
    Therefore, we conclude the proof of the lemma. 
\end{proof}

\subsection{Proof of Lemma \ref{lm:mdp-delta}}\label{sec:delta-proof}
\begin{proof}
    
In the following lemma, we firstly bound the difference between the $Q$-function maintained in Algorithm \ref{alg:mdp} (without bonus) and the real $Q$-function of any policy $\pi$ by their expected difference at next step, plus an error term. This error term can be upper bounded by our bonus with high probability. This lemma can be derived from Lemma B.4 in \cite{jin2020provably} with slight revisions.

\begin{lemma}\label{lm:value-bound}
Set $\lambda=\max\{1,\Psi^2\}$. There exists an absolute constant $c_\rho$ such that for $\rho=c_\rho d_\psi H \sqrt{\iota}$ where $\iota=\log(2d_\psi \Psi^2 T/\delta)$, and for any fixed policy $\pi$, with probability at least $1-\delta/2$, we have for all $s\in \cS$, $\omega \in \Omega$, $a \in \cA$, $h\in[H]$, $t\in[T]$,
\begin{equation*}
    \psi(s,\omega,a)^\top q_h^t -Q_h^\pi(s,\omega,a)=P_h(V_{h+1}^t-V_{h+1}^\pi)(s,\omega,a)+\triangle_h^t(s,\omega,a),
\end{equation*}
for some $\triangle_h^t(s,\omega,a)$ that satisfies $\vert \triangle_h^t(s,\omega,a) \vert \leq \rho \Vert \psi(s,\omega,a)\Vert_{(\Gamma_h^t)^{-1}}$.
\end{lemma}

Now we are ready to prove  Lemma \ref{lm:mdp-delta}. 
By the  definition of $\delta_h^t$ in (\ref{eq:defi-delta}), we have $\delta_h^t=(v_h+P_h V_{h+1}^t-Q_h^t)=(P_h V_{h+1}^t-P_h V_{h+1}^{\pi^t})+(Q_h^{\pi^t}-Q_h^t)$. Therefore, by the construction of $Q_h^t$ in Algorithm \ref{alg:mdp}, we obtain that
\begin{align*}
    \delta_h^t(s,\omega,a)&\geq (P_h V_{h+1}^t-P_h V_{h+1}^{\pi^t})(s,\omega,a)+Q_h^{\pi^t}(s,\omega,a)-\left(\psi(s,\omega,a)^\top q_h^t+\rho \Vert \psi(s,\omega,a)\Vert_{(\Gamma_h^t)^{-1}} \right)\\
    &= -\triangle_h^t(s,\omega,a)-\rho \Vert \psi(s,\omega,a)\Vert_{(\Gamma_h^t)^{-1}}
    \geq -2\rho \Vert \psi(s,\omega,a)\Vert_{(\Gamma_h^t)^{-1}},
\end{align*}
which concludes the proof.
\end{proof}

\subsection{Proof of Lemma \ref{lm:mdp-gap}}
\begin{proof}
Denote the optimal signaling schemes corresponding to the real prior $\mu_h^*$ and the estimated prior $\mu_h^t$ respectively as
\begin{equation*}
\pi_h'=\argmax_{\pi_h\in\Ps(\mu_h^*)} \big\langle Q_h^t, \mu_h^*\otimes\pi_h\big\rangle_{\Omega\times\cA}(\cdot;C^t) ~~~\text{and}~~~ \pi_h''=\argmax_{\pi_h\in\Ps(\mu_h^t)} \big\langle Q_h^t, \mu_h^t\otimes\pi_h \big\rangle_{\Omega\times\cA}(\cdot;C^t),    
\end{equation*}
where the $Q$-function $Q_h^t$ is given by Algorithm \ref{alg:mdp}.  Notably, $\pi_h'$ is different from the truly optimal policy $\mu_h^*$, since $\pi_h'$ is computed based on the approximate $Q$-function $Q_h^t$. By definition, we can decompose the difference as follows:
\begin{align}
    \label{eq:Q-optimism}
    \big\langle Q_h^t,\mu_h^*\otimes\pi_h^*-\mu_h^t\otimes \pi_h^t \big\rangle_{\Omega\times\cA}(s_h;C^t)
    =& \big\langle Q_h^t,\mu_h^*\otimes\pi_h^*-\mu_h^*\otimes \pi_h'\big\rangle_{\Omega\times\cA}(s_h;C^t)\\
    \label{eq:Gap-OPT}
    &~~~~+\big\langle Q_h^t,\mu_h^*\otimes\pi_h'-\mu_h^t\otimes \pi_h''\big\rangle_{\Omega\times\cA}(s_h;C^t)\\
    \label{eq:Gap-bound}
    &~~~~+\big\langle Q_h^t,\mu_h^t\otimes\pi_h''-\mu_h^t\otimes \pi_h^t\big\rangle_{\Omega\times\cA}(s_h;C^t) .
\end{align}
By definition,  equation \eqref{eq:Q-optimism} is always non-positive. Apply Lemma \ref{lm:opt-gap} to equation \eqref{eq:Gap-OPT} and we can get
\begin{align*}
    \big\langle Q_h^t,\mu_h^*\otimes\pi_h'-\mu_h^t\otimes \pi_h'' \big\rangle_{\Omega\times\cA}(s_h;C^t) \leq& \Gap\bigg(s_h,\mu_h^*(\cdot\vert c_h^t),\uB_1\big(\mu_h^*(\cdot\vert c_h^t),\big\Vert \mu_h^*(\cdot\vert c_h^t)-\mu_h^t(\cdot\vert c_h^t)\big\Vert_1\big);Q_h^t\bigg)\\
    &~~~~+\frac{H}{2}\big\Vert \mu_h^*(\cdot\vert c_h^t)-\mu_h^t(\cdot\vert c_h^t) \big\Vert_1.
\end{align*}
According to Corollary \ref{cor:gap-linear}, we can bound the above equation with the norm of feature vector and the radius of confidence region for $\theta_h$.
\begin{align*}
    \big\langle Q_h^t,\mu_h^*\otimes\pi_h'-\mu_h^t\otimes \pi_h'' \big\rangle_{\Omega\times\cA}(s_h;C^t) 
    \leq  \big ( \frac{H L_\mu K}{p_0D}+\frac{H L_\mu K}{2}\big)\beta \Vert \phi(c_h^t) \Vert_{(\Sigma_h^t)^{-1}}.
\end{align*}

We also note that equation \eqref{eq:Gap-bound} is equal to $\Gap\big(s_h,\mu_h^t(\cdot\vert c_h^t),\mu_{\cB_h^t}(\cdot\vert c_h^t);Q_h^t\big)$. By Lemma \ref{lm:gap-for-est}, on the event $\{\theta_h^*\in\cB_h^t\}$, we have
\begin{equation*}
\Gap\big(s_h,\mu_h^t(\cdot\vert c_h^t),\mu_{\cB_h^t}(\cdot\vert c_h^t);Q_h^t \big) \leq \frac{2H L_\mu K}{p_0 D}\beta \Vert \phi(c_h^t)\Vert_{(\Sigma_h^t)^{-1}}.
\end{equation*}
Therefore, on the given event, we have
\begin{equation*}
     \mathbb{E}_{\mu_h^*,\pi_h^*} \big[\big\langle Q_h^t,\mu_h^*\otimes\pi_h^*-\mu_h^t\otimes \pi_h^t \big\rangle_{\Omega\times\cA}(s_h;C^t)\vert s_1=s_1^t\big] \leq \bigg( \frac{3H L_\mu K}{p_0D}+\frac{H L_\mu K}{2}\bigg) \beta \Vert \phi(c_h^t)\Vert_{(\Sigma_h^t)^{-1}}.
\end{equation*}
Summing up together, we get
\begin{align*}
    \sum_{t\in[T]}\sum_{h\in[H]}\mathbb{E}_{\mu_h^*,\pi_h^*} \big[\big\langle Q_h^t,\mu_h^*\otimes\pi_h^*&-\mu_h^t\otimes \pi_h^t \big\rangle_{\Omega\times\cA}(s_h;C^t)\vert s_1=s_1^t\big]\\ &\leq   \bigg( \frac{3H L_\mu K}{p_0D}+\frac{H L_\mu K}{2} \bigg) \beta \sum_{t\in[T]}\sum_{h\in[H]} \Vert \phi(c_h^t)\Vert_{(\Sigma_h^t)^{-1}}.
\end{align*}
Therefore, we conclude the proof of Lemma  \ref{lm:mdp-gap}. 
\end{proof}

\subsection{Proof of Lemma \ref{lm:mdp-concentration}}
\begin{proof}
By definition, we can rewrite the difference in Lemma \ref{lm:mdp-concentration} as
\begin{align*}
    \big\langle Q_h^t,(\mu_h^t-\mu_h^*)\otimes\pi_h^t\big\rangle_{\Omega\times\cA}(s_h^t;C^t) &=  \int_{\Omega\times\cA} \big[\mu_h^t(\omega\vert c_h^t)-\mu_h^*(\omega\vert c_h^t)\big]\pi_h^t(s,\omega,a)Q_h^t(s,\omega,a)\ud a\ud\omega\\
    &=  \int_\Omega \big[\mu_h^t(\omega\vert c_h^t)-\mu_h^*(\omega\vert c_h^t)\big]\int_\cA \pi_h^t(s,\omega,a)Q_h^t(s,\omega,a)\ud a\ud\omega.
\end{align*}
By Holder's inequality, we have
\begin{equation*}
    \bigg\vert \big\langle Q_h^t,(\mu_h^t-\mu_h^*)\otimes\pi_h^t\big\rangle_{\Omega\times\cA}(s_h^t;C^t) \bigg\vert \leq \big\Vert \mu_h^t(\cdot\vert c_h^t)-\mu_h^*(\cdot\vert c_h^t) \big\Vert_1 \sup_{\omega\in\Omega} \bigg\vert \int_\cA \pi_h^t(s,\omega,a)Q_h^t(s,\omega,a)\ud a \bigg\vert.
\end{equation*}
Since $Q_h^t \leq H$ for any $h\in[H]$ and $t\in[T]$, the inequality can be simplified to
\begin{equation*}
    \bigg\vert \big\langle Q_h^t,(\mu_h^t-\mu_h^*)\otimes\pi_h^t\big\rangle_{\Omega\times\cA}(s_h^t;C^t) \bigg\vert \leq H \big\Vert \mu_h^t(\cdot\vert c_h^t)-\mu_h^*(\cdot\vert c_h^t) \big\Vert_1. 
\end{equation*}
With the assumption of the prior and link function, on the given event, we obtain that
\begin{equation*}
    \sum_{t\in[T]}\sum_{h\in[H]}\big\langle Q_h^t,(\mu_h^t-\mu_h^*)\otimes\pi_h^t\big\rangle_{\Omega\times\cA}(s_h^t;C^t) \leq H L_\mu K \beta \sum_{h\in[H]} \sum_{t\in[T]} \Vert \phi(c_h^t)\Vert_{(\Sigma_h^t)^{-1}}.
\end{equation*}
Therefore, we conclude the proof of Lemma \ref{lm:mdp-concentration}. 
\end{proof}

\subsection{Proof of Corollary \ref{cor:gap-linear}}
\begin{proof}
According to Assumption \ref{assum:prior} for the prior, we can show that for any $\mu_{\theta'}(\cdot | c) \in \cB$, 
$$\norm{\mu_\theta(\cdot | c)  - \mu_{\theta'}(\cdot | c) }_{1} \leq L_\mu \big\Vert f(\phi(c)^\top\theta)-f(\phi(c)^\top\theta') \big\Vert. $$
Moreover, by Assumption \ref{assum:f} for the link function $f(\cdot)$, we have
$$\norm{\mu_\theta(\cdot | c)  - \mu_{\theta'}(\cdot | c) }_{1}\leq L_\mu K \big\Vert\phi(c)^\top (\theta-\theta') \big\Vert \leq L_{\mu}K \norm{\phi(c)}_{\Sigma^{-1}} \epsilon.$$
Therefore, $\cB \subseteq \uB(\mu_{\theta}(\cdot|c), L_{\mu}K \Vert\phi(c)\Vert_{\Sigma^{-1}} \epsilon )$, and by Lemma \ref{lm:robustness-gap}, we can conclude the result.
\end{proof}

\subsection{Auxiliary Lemmas}

This section presents several auxiliary lemmas and their proofs. 

\begin{lemma}[Martingale Bound; \cite{cai2020provably}] \label{lm:mdp-zeta}
For $\zeta_{t,h}^1$ and $\zeta_{t,h}^2$ defined in (\ref{eq:defi-zeta}) and for any fixed $\delta\in(0,1)$, with probability at least $1-\delta/2$, we have
\begin{equation*}
    \sum_{t\in[T]}\sum_{h\in[H]}(\zeta_{t,h}^1+\zeta_{t,h}^2) \leq \sqrt{16TH^3\log(4/\delta)}.
\end{equation*}
\end{lemma}

\begin{proof}
    See \cite{cai2020provably} for a detailed proof. 
\end{proof} 

\begin{lemma}\label{lm:sum-of-phi}
  Suppose that $\phi_1,\phi_2,\ldots,\phi_{T}\in R^{d_\phi\times d}$ and for any $1 \leq i \leq T$, there exists a constant $\Phi>0$ such that $\Vert\phi_i\Vert \leq \Phi.$ Let $\Sigma_t=\lambda I_{d_\phi}+\sum_{i\in[t-1]}\phi_i\phi_i'$ for some $\lambda\geq \Phi^2$. Then,
  \begin{equation*}
      \sum_{t\in[T]} \Vert \phi_t \Vert_{(\Sigma_t)^{-1}} \leq \sqrt{2d_\phi T\log(1+T\Phi^2/(\lambda d_\phi))}.
  \end{equation*}
\end{lemma}
\begin{proof}
    Firstly, we apply Cauchy-Schwartz inequality,
    \begin{equation*}
        \sum_{t\in[T]} \Vert \phi_t \Vert_{(\Sigma_t)^{-1}} \leq \sqrt{T\sum_{t\in[T]} \Vert \phi_t \Vert^2_{(\Sigma_t)^{-1}}}.
    \end{equation*}
    Since $\Vert \phi_t \Vert_{(\Sigma_t)^{-1}}=\sqrt{\phi_t^\top(\Sigma_t)^{-1}\phi_t}\leq\sqrt{\lambda^{-1}\phi_t^\top\phi_t}\leq \Phi/\sqrt{\lambda}\leq 1$, we can use Lemma \ref{lm:Feature bound} to bound the sum of squares:
    \begin{align*}
        \sum_{t\in[T]} \Vert \phi_t \Vert_{(\Sigma_t)^{-1}} &\leq \sqrt{2T\log(\det(\Sigma_T)\det(\Sigma_1)^{-1})}\\
        &\leq \sqrt{2d_\phi T\log(1+T\Phi^2/(\lambda d_\phi))}.
    \end{align*}
    The last inequality is derived from Lemma \ref{lm:det inequality}.
\end{proof}

\begin{lemma}[{Sum of Potential Function; \cite{agarwal2019reinforcement}}]\label{lm:Feature bound}
	For any sequence of $\{\phi_t\}_{t\in[T]}$, let $\Sigma_t=\lambda I_h+\sum_{t\in[t-1]}\phi_i\phi_i'$ for some $\lambda\ge 0$. Then we have
	\begin{equation*}
	\sum_{t \in [T]}\min\{\Vert\phi_t\Vert_{(\Sigma_t)^{-1}}^2,1\} \leq 2\log(\det(\Sigma_T)\det(\Sigma_1)^{-1}).
	\end{equation*}
\end{lemma}

\begin{proof}
    See \cite{agarwal2019reinforcement} for a detailed proof. 
\end{proof}

\begin{lemma}[Determinant-Trace Inequality]\label{lm:det inequality}
	Suppose that $\phi_1,\phi_2,\ldots,\phi_{T}\in R^{d_\phi\times d}$ and for any $1 \leq i \leq T$, there exists a constant $\Phi>0$ such that $\Vert\phi_i\Vert \leq \Phi.$ Let $\Sigma_t=\lambda I_{d_\phi}+\sum_{i\in[t-1]}\phi_i\phi_i'$ for some $\lambda\ge 0$. Then,
	\begin{equation*}
	\det(\Sigma_t)\leq\big(\lambda+ t\Phi^2  /  d_\phi \big )^{d_\phi}.
	\end{equation*}	
\end{lemma}
\begin{proof} 
	Let $\lambda_1,\lambda_2,\ldots,\lambda_h$ be the eigenvalues of $\Sigma_t$. Since $\Sigma_t$ is positive definite, its eigenvalues are positive. Also, note that $\det(\Sigma_t)=\prod_{s=1}^{d_\phi} \lambda_s$ and $\Tr(\Sigma_t)=\sum_{s=1}^{h}\lambda_s$. By inequality of arithmetic and geometric means
	\begin{equation*}
	\det(\Sigma_t)\leq(\Tr(\Sigma_t)/d_\phi)^{d_\phi}.
	\end{equation*}
	It remains to upper bound the trace:
	\begin{equation*}
	\Tr(\Sigma_t)=\Tr(\lambda I_{d_\phi})+\sum_{i=1}^{t-1} \Tr(\phi_i\phi_i')=d_\phi\lambda+\sum_{i=1}^{t-1} \Vert \phi_i\Vert^2\leq d_\phi\lambda+t\Phi^2
	\end{equation*}
	and the lemma follows.
\end{proof}

\end{document}